\documentclass[12pt]{article}

\usepackage{amsmath,amssymb,amsthm}
\usepackage{xcolor}

\usepackage[utf8x]{inputenc}

\usepackage{nameref,hyperref}

\usepackage{amsfonts,algorithm2e}
\usepackage{graphicx}
\usepackage{booktabs}

\usepackage{subcaption}
\usepackage{authblk}

\usepackage[resetlabels]{multibib}
\newcites{app}{Appendix References}

\usepackage[markup=default,final]{changes}
\setdeletedmarkup{{\color{red}\sout{#1}}}
\usepackage[textsize=tiny]{todonotes}

\newcommand{\argmax}{\operatornamewithlimits{argmax}}
\newcommand{\argmin}{\operatornamewithlimits{argmin}}
\renewcommand{\Re}{{\mathbb{R}}}
\renewcommand{\Pr}{{\mathbb{P}}}
\newcommand{\Ex}{{\mathbb{E}}}
\newcommand{\Cov}{\mathrm{Cov}}
\newcommand{\tr}{{\textup{trace}}}

\newcommand{\tA}{{\widetilde A}}
\newcommand{\tB}{{\widetilde B}}

\newcommand{\PM}{{\mathcal{P}}}
\newcommand{\DSM}{{\mathcal{D}}}
\newcommand{\AM}{{\mathcal{A}}}
\newcommand{\GM}{{\mathcal{M}}}
\newcommand{\1}{{\mathbf{J}}}
\newcommand{\0}{{\mathbf{0}}}
\newcommand{\ErdosRenyi}{{Erd{\H o}s-R{\'e}nyi}}
\newcommand{\CER}{{\mathrm{CorrER}}}

\newtheorem{theorem}{Theorem}
\newtheorem{proposition}[theorem]{Proposition}

\theoremstyle{definition}
\newtheorem{definition}[theorem]{Definition}
\theoremstyle{remark}
\newtheorem{remark}[theorem]{Remark}

\begin{document}

\title{Matched Filters for Noisy Induced Subgraph Detection}

\author[$\dag$]{Daniel~L.~Sussman}
\author[$\ast$]{Youngser~Park}
\author[$\ast$]{Carey~E.~Priebe}
\author[$\ddag$]{Vince~Lyzinski}

\affil[$\dag$]{\small Department of Mathematics and Statistics, Boston University}
\affil[$\ddag$]{\small Department of Mathematics, University of Maryland College Park}
\affil[$\ast$]{\small Department of Applied Mathematics and Statistics, Johns Hopkins University}

\maketitle

\begin{abstract}
The problem of finding the vertex correspondence between two noisy graphs with different number of vertices where the smaller graph is still large has many applications in social networks, neuroscience, and computer vision.
We propose a solution to this problem via a graph matching matched filter: centering and padding the smaller adjacency matrix and applying graph matching methods to align it to the larger network.
The centering and padding schemes can be incorporated into any algorithm that matches using adjacency matrices.
Under a statistical model for correlated pairs of graphs, which yields a noisy copy of the small graph within the larger graph, the resulting optimization problem can be guaranteed to recover the true vertex correspondence between the networks.
 However, there are currently no efficient algorithms for solving this problem.
To illustrate the possibilities and challenges of such problems, we use an algorithm that can exploit a partially known correspondence and show via varied simulations and applications to {\it Drosophila} and human connectomes that this approach can achieve good performance.
\end{abstract}

\maketitle

\section{Introduction}

In many settings, we often want to quantify how multiple networks relate to each other in order to study how actors jointly use these networks.
This may arise from multiple modalities, such as communications networks, delivery networks, financial networks, and social networks, or from a time dynamic setting.
Similarly, in neuroscience or biology we may seek to compare brain networks or protein networks of different individuals or species.
In computer vision, we may seek to determine the relationships between the shapes and objects in two images as represented by graphs.
Often, these networks are on different unmatched sets of vertices that are not the same cardinality.
This limits the set of available tools as the adjacency matrices and other graph statistics may not be directly comparable.

In a related fashion, detection and location of particular induced subgraphs that correspond to a given activity or structure of interest are also paramount problems in various domains.
Examples include determining whether a particular activity pattern is present in a social network, detecting certain shapes in an image, or discovering motifs in brain networks.

When the induced subgraphs are not too large, with up to $\approx$10--20 vertices, there are a number of approaches including tree search approaches~\cite{Carletti2018-ea} and backtracking algorithms~\cite{Kuramochi2001-ll}.
Parallel algorithms have also been proposed to approximately enumerate the number of isomorphic copies of a subgraph appearing in a large network~\cite{Slota2013-zg}.
However, it may be the case that these subgraphs contain $\gg20$ vertices. 
In this case, many existing approaches will either fail to find the subgraphs of interest or be computationally intractable.
While state space representations~\cite{Cordella2004-nc} have been used to identify these larger subgraphs, these also have limitations in terms of the applicable graph structure.
As a simple example of the challenge of this problem, consider finding or enumerating cliques.
Finding cliques and estimating the number of cliques of small sizes can be achieved via fast algorithms~\cite{Schmidt2009-nf}, but the problem of finding the maximum clique is a notoriously challenging problem~\cite{Bomze1999-bx}.

Furthermore, while even the subgraph isomorphism is itself NP-complete~\cite{Ullmann1976-ir}, we often expect the subgraphs may not appear exactly in the graph but rather only approximately.
This may be due to errors and missingness in the observation of one or both graphs, so finding an exact subgraph might not be possible.

We view both of these as different points on a continuum of related problems that are all examples of graph matching~\cite{ConteReview,foggia2014graph}. 
In a simpler setting, when the two graphs adjacency matrices $A,B\in\{0,1\}^{n\times n}$ are of equal size, the graph matching problem (GMP) seeks the permutation matrix 
\begin{align}
\label{eq:GM}
  \argmin_{P\in \PM} \|A-PBP^T\| = \argmax_{P\in \PM} \tr(APBP^T),
\end{align}
where $\PM$ is the set of $n\times n$ permutation matrices. 
While this problem is NP-hard in general, there are numerous approaches in the literature that attempt to approximately solve the GMP. 
\cite{ConteReview} and \cite{Emmert-Streib2016-st} provide reviews of the prescient literature.
In particular, when prior knowledge about the correspondence between vertices the  can be incorporated into the algorithms, the GMP can be approximately solved efficiently for graphs with more than $10^5$ vertices~\cite{lyzinski_spectral,yartseva2013performance} without the need for sophisticated modern parallel computing to be brought to bear.

A more challenging problem that we will not consider is to detect anomalous subgraphs within a collection of graphs~\cite{Akoglu2015-xe}.
In the anomaly detection setting, the structure of the anomalous subgraph may be only known up to certain graph characteristics or deviations from the structure of the remaining graph.
In our setting, the approximate topology of the graph to be detected is known explicitly.

To approximately solve the noisy induced subgraph problem, we will use the machinery of graph matching to construct a {\em graph matching matched filter}, aligning the smaller network to its closest subgraph within the larger network.
Our main goals in this manuscript are to investigate the theoretical limitations of this noisy graph matching framework when the graphs may be of very different sizes.
To match graphs of radically different sizes in Eq.~\eqref{eq:GM}, we consider padding the smaller matrix to render graph matching an appropriate tool for the problem~\cite{FAP}.
Under a statistical model for noisily implanting an induced subgraph into a larger network, we show that the true induced subgraph will be found by brute force GM algorithm using an appropriate padding scheme, provided graph correlations and probabilities satisfy certain parameter assumptions.

We further demonstrate the effectiveness of these strategies when the vertex correspondence is partially known.
By relaxing the graph matching problem constraints to the set of doubly stochastic matrices, a gradient descent approach can be put to use~\cite{FAQ}.
The problem is non-convex~\cite{Lyzinski2016-kp} and hence the initialization can incorporate the partial correspondence and random restarts to seek multiple local minimums.
By choosing the optimal of these local minima, an exact or approximate recovery of the vertex correspondence can often be found.

\section{Background for Graph Matching}

In this section we provide a brief background on graph matching, some methods to incorporate prior information, and a statistical model for correlated graphs.
Throughout the remainder of this article we will use the following notation.

Let $[n]=\{1,2,\dotsc,n\}$.
Let $\PM_n$ and $\DSM_n$ denote the set of $n\times n$ permutation matrices and doubly stochastic matrices, respectively.
Let $\1_n$ and $\0_n$ denote the $n\times n$ all-ones and all-zeros matrices, respectively, and $\1_{nm}$ and $\0_{nm}$ for $n\times m$ rectangular versions.
Let $\AM_n$ denote the set of adjacency matrices corresponding to simple undirected graphs, i.e.\ the set of symmetric hollow $\{0,1\}$-valued $n\times n$ matrices.
When clear from context, we may omit subscripts indicating the sizes of these matrices.
Finally, let $\oplus$ denote the direct sum of two matrices, $M\oplus M'= \begin{pmatrix}  M & \0 \\ \0 & M' \end{pmatrix}$.

\subsection[Contrasts with Subgraph Isomorphism]{Contrasts with Subgraph Isomorphism}

Before describing our noisy matched filter and related procedures, we would like to distinguish our approach and goals from that of exact subgraph isomorphism.
The problem of exact subgraph isomorphism is to find an induced subgraph of the larger graph which is isomorphic to the smaller graph.
Frequently, the goal is to find all isomorphic copies of the smaller graph within the larger graph.
Examples of this could be counting triangle, $k$-cliques, and $k$-cycles, some of which have fast algorithms associated with them while in general the problem is NP-complete~\cite{Ullmann1976-ir}.

The subgraph isomorphism problem is an area of active research~\cite{Samsi2017-if, especially in areas such as computer vision and pattern recognition.
Recent algorithm developments include the VF3 algorithm~\cite{Carletti2018-ea}, which builds on the VF2 algorithm~\cite{Cordella2004-ij}, and both employ a tree-search approach.
Other popular and recent algorithms include the LAD algorithm~\cite{treesearch1} and the RI algorithm~\cite{Bonnici2013-kn}.}

The problem of noisy matched filters or inexact subgraph isomorphism differs from exact isomorphism in that only the ``closest'' induced subgraph is sought after.
Indeed, for most of the problems we consider below, an isomorphic copy of the subgraph is likely to not even exist in the larger graph (except for our correlation $\rho = 1$ settings).
For this reason, exact subgraph isomorphism approaches cannot succeed.
Research for inexact subgraph isomorphism is even more nascent with only a few recent approaches having been developed~\cite{Caelli2004-nd,Dutta2017-qg,Carletti2016-yn}, each of which has a distinct approach and objective including spectral and higher-order contextual approaches.

A second differing emphasis of this work is that the overall goal is to understand when it is possible to find a specific set of vertices within the larger graph that correspond to the vertices in the small graph.
When many isomorphic copies of the small graph are present in the large graph, it will be impossible to distinguish which, if any, of these copies is the ``true'' vertex-correspondence.}
 
As such, since our goals are different from those of exact subgraph isomorphism in two ways, comparisons to exact subgraph isomorphism algorithms are not made.
Additionally, our theoretical results are algorithm-independent, applying to any approach which attempts to optimize the given objective functions.

\subsection[Graph Matching Algorithms]{Graph Matching Algorithms}

Solving the GMP problem exactly is NP-complete in general, but there are a number of approaches which provide good approximate solutions~\cite{ConteReview,Emmert-Streib2016-st}.
Some approaches rely on tree-based methods for exact branch and bound algorithms for integer programs~\cite{Ghahraman1980-pg}.
While these approaches are guaranteed to find a global minimum, their computational scaling is very poor.
For larger graphs, the constraints for the GMP are often relaxed so that continuous optimization machinery can be deployed~\cite{FAQ,zaslavskiy2009path,jovo}.
The relaxed solutions are then projected back onto $\PM$ yielding an approximate solution to the graph matching problem.
Some relaxations involve applying spectral methods which enables the application of fast linear algebra techniques~\cite{Egozi2013-jh}.
A convex approach has also been proposed but requires the graph to have certain properties~\cite{alex}, which do not frequently hold in practice.

For the computational experiments in this manuscript we will rely on a principled indefinite relaxation of the GMP constraints~\cite{Lyzinski2016-kp} to the set of doubly stochastic matrices $\DSM{}$, the convex hull of $\PM$.
Details of this algorithm are discussed in Section~\ref{sec:computation}.
This approach has been shown to perform very well in practice on a number of benchmark problems, especially when exploiting multiple random restarts to avoid poor performing local minima~\cite{FAQ}.
We expect that many of the phenomena observed in our experiments would appear in similar fashions for other algorithms.

Frequently, these approaches can exploit \emph{seeded vertices}, i.e., a partial list of known correspondences between the two vertex sets.
When sufficiently many seeds are present, these algorithms, which often have few guarantees without seeds, can be solved efficiently and the resulting match can be guaranteed to be correct~\cite{FAP,JMLR:v15:lyzinski14a} asymptotically almost surely for relatively general random graph models.
While the theory we discuss below does not require seeds, our algorithms will use seeds and the algorithmic performance relies heavily on a sufficient number of a priori available seeds.

In this work, we employ a soft seeding approach~\cite{Fang2018-er} which uses the prior information to initialize the algorithm but does not insist that the correspondence for seeds be fixed; i.e., if the seeded vertices are $[m]$ in both networks, we initialize the gradient descent graph matching algorithm at $I_m\oplus D$ for suitable $D$ in $\mathcal{D}_{n-m}$. 
Conversely, in the hard-seeded setting we optimize over $\Pi_{n-m}$ with the global solution being then of the form $I_m\oplus P$ for $P\in\Pi_{n-m}.$ 
This allows for the possibility that seeds may be noisy, allowing for the algorithm to potentially correct for incorrect initial seeds, a scenario that we do not explore further here.

\subsection{Statistical Models}

In order to study the applicability and limitations of a graph matching approach for subgraph detections, we analyze the problem from a statistical viewpoint in the setting of the correlated heterogeneous \ErdosRenyi{} model~\cite{Lyzinski2016-kp}.
The algorithmic approaches we use in our empirical investigations are only partially inspired by this model.
The model enforces that the objective functions considered are not unreasonable but, except for the oracle approach, the objective functions are based primarily on topological considerations.

The following definition provides a distribution for pairs of random graphs with different numbers of vertices where there is a positive correlation between corresponding vertex-pairs.
In particular, the marginal probabilities of edges in each graph can vary and the presence of an edge in one graph increases the conditional probability of the corresponding edge in the second graph.
In a social network setting, this can be interpreted as two people who are related in one graph are more likely to be related in the other graph.
Without loss of generality, we assume that the smaller graph (of order $n_c$) is denoted $A$ and that the corresponding vertices to $A$ in $B$ are the first $n_c \leq n$ vertices.

\begin{definition}[Correlated \ErdosRenyi{}]
Suppose $\Lambda \in [0,1]^{n\times n}$ and $R\in [0,1]^{n_c\times n_c}$ for $0<n_c\leq n$.
Denote by $\Lambda^{c}$, the order $n_c$ principal submatrix of $\Lambda$.
A pair of adjacency matrices $(A,B)\sim \CER(\Lambda, R)$ if $A\in \AM{}_{n_c}$, $B\in \AM{}_n$, and for each $u<v$, $B_{uv}$ are independent with $B_{uv}\sim \mathrm{Bernoulli}(\Lambda_{uv})$ and for each $u<v\leq n_c$, $A_{uv}$ are independent with $A_{uv}\sim \mathrm{Bernoulli}(\Lambda_{uv})$.
Additionally, for each $u,v, u',v'\in[n_c]$ with $u<v$, $u'<v'$, $u\neq u'$, $v\neq v'$ , $B_{u,v}$ and $A_{u',v'}$ are independent except that $\mathrm{corr}(A_{uv}, B_{uv}) = R_{uv}$.
\end{definition}

When $n_c=n$, it can be shown that the solution to the GMP will asymptotically almost surely yield the correct vertex correspondence under mild assumptions on $\Lambda$ and $R$, i.e., the unique element in the argmin in Eq.~\eqref{eq:GM} is the identity matrix $I$~\cite{JMLR:v15:lyzinski14a,Lyzinski2016-kp}.

Note that other authors have considered equivalent reparameterized versions of the model above.
In particular, a common viewpoint is that each observed graph is a random subgraph of some common underlying graph, which also arises from independent edge sampling~\cite{yartseva2013performance}.
By varying the parameters for the underlying graph and the parameters for the random subgraph, any correlated \ErdosRenyi{} can be represented in this way, provided all correlations are non-negative.
The reverse is also true, and more distributions can be constructed by allowing for negative correlations and non-identical distributions.
Another closely related possibility is that one graph is a noisy observation of an induced subgraph from the other graph, where there may be extra or missing edges~\cite{Lyzinski2016-kp}.
This can also be related to the correlated \ErdosRenyi{} by an appropriate reparameterization.

The marginal distribution of each graph is an independent edge graph which is closely related to ideas of exchangeable random graphs and latent position models which have been used extensively in various application areas~\cite{Hoff2009-pj,Hoff2002-lj}.
Any random graph, including exchangeable random graphs, can be viewed as a mixture of independent edge random graphs~\cite{Durante2016-tk}. 
Viewed from this perspective, much of the theory we consider below can be extended to such mixtures presuming that the assumptions hold with sufficiently high probability across mixture components.

\section{Padding Approaches}

In order to match pairs of graphs with differing numbers of vertices, we propose to pad the smaller graph/matrix with enough rows and columns to match the order of the larger matrix.
We will consider a trio of padding schemes which will result in differing qualities for the resulting match~\cite{FAP}:
{\em Naive Padding} Let $\tA=A\oplus \0_{n_j}$ and match $\tA$ and $B$;
{\em Centered Padding} Let $\tA=(2A-\1)\oplus \0_{n_j}$, let $\tB=2B-\1$, and match $\tA$ and $\tB$;
{\em Oracle Padding} Let $\tA=(A-\Lambda_{n_c})\oplus \0$, let $\tB=B-\Lambda$, and match $\tA$ and $\tB$.

As we will see in the next section, the naive padding scheme---which finds the best fitting subgraph of $B$ to match with $A$---will not be guaranteed to find the true correspondence between nodes.
In contrast, the centered padding scheme---which finds the best fitting induced subgraph of $B$ to match with $A$---and the oracle padding scheme are guaranteed to succeed under mild model conditions, even in the presence of an exponentially small (in terms of the size of $B$) subgraph $A$.
In general, the oracle padding scheme will be inaccessible as $\Lambda$ is unknown, but using various methods to estimate $\Lambda$~\cite{chatterjee2014matrix,Davenport2014-tf}, we can approximate $\Lambda$ in ways that can improve matching performance.
Importantly, the padding and centering schemes are not tied to the algorithms described Section~\ref{sec:computation}, and can be used with any graph matching approach which uses the adjacency matrices.

Note that each of these padding approaches corresponds to a particular graph detection problem.
With the naive padding scheme, the optimization program is equivalent to the problem $\argmax_{\sigma:[n_c]\rightarrowtail [n]} \sum_{u,v} A_{\sigma(u)\sigma(v)}B_{uv}$, where $\rightarrowtail$ denotes that $\sigma$ is injective.
Hence, the global optimum reveals the largest common subgraph of $B$ and $A$, as only edge agreements are rewarded whereas edge disagreements are not penalized. 
Note that if $B$ has a large clique, then the naive approach will tend to match $A$ to this large clique as all edges in $A$ are present in the large clique.
As we will see below, the naive padding approach can be guaranteed with high probability to yield an incorrect vertex correspondence under only mild assumptions on $R$ and $\Lambda$.

If the graphs have the same number of nodes, both the centered padding scheme and the naive padding scheme are equivalent~\cite{Bunke1997-bk}, while they are not equivalent for graphs of different orders.
With the centered padding scheme, the optimization program is equivalent to the problem $$\argmin_{\sigma:[n_c]\rightarrowtail [n]} \sum_{u,v} \vert A_{\sigma(u)\sigma(v)} - B_{uv} \vert.$$
In this case, the global optimum reveals the induced subgraph of $B$ which is nearest to $A$ in terms of edge-edit distance~\cite{Bunke1997-bk}.
As will be seen below, the optimum to this program can be guaranteed to correspond the true vertex correspondence provided that the graphs are not too sparse and the correlation is sufficiently large.

The oracle padding scheme corresponds to the problem, 
$$\argmax_{\sigma:[n_c]\rightarrowtail [n]} \sum_{u,v} \widehat{\Cov}(A_{\sigma(u)\sigma(v)} B_{uv}),$$ where $\widehat{\Cov}$ denotes the empirical covariance when both means are known, $\widehat{\Cov}(A_{ij},B_{uv}) = (A_{ij}-\Lambda_{ij}) (B_{uv} - \Lambda_{uv})$.
As this approach eliminates any undue rewards from matching likely edges to even more likely edges or unlikely edges to even more unlikely edges, the global optimum is guaranteed with high probability to be the true correspondence even under mild assumptions on $\Lambda$ and $R$.
Indeed, even if $A$ and $B$ were non-identically distributed, with means $\Lambda_A$ and $\Lambda_B$ respectively, by matching $A-\Lambda_A$ and $B-\Lambda_B$, the theoretical results will still hold~\cite{lyzinski2017graph}.

While the theory below shows that under the correlated heterogeneous \ErdosRenyi{} model the oracle padding yields stronger guarantees than the centered padding method and the naive padding method can have very bad performance, this will not necessarily translate in practice.
This can occur especially due to the fact that we do not find the true global optimum of the GMP but rather a local optimum near our initialization.
Our simulations and real data examples illustrate how centering or naive may be preferable in certain scenarios, and biased estimates of the oracle may also improve performance.
In practice, one may try different paddings and even use the solution of one padding approach to initialize the optimization for another padding approach.
See Section~\ref{sec:experiments} for examples where various approaches, including matching for a second time, may provide more or less ideal performance.

\subsection{Theory}

For each of the padding scenarios, we will consider $n_c$ and $n$ as tending to $\infty$ in order to understand the ability of these optimization programs to detect progressively larger subgraphs.
Note that we will require that the number of vertices $n_c$ is growing with $n$.
Indeed, if $n_c$ is fixed and $n$ grows then eventually every subgraph of size $n_c$ will likely appear multiple times as an induced subgraph in $B$ just by chance~\cite{Lovasz2012-df}.

For each padding scenario let $P^*\in \{0,1\}^{n_c\times n_c}$ denote the order $n_c$ principal submatrix of the solution to the corresponding graph matching problem.
The proofs of each theorem below can be found in Appendix~\ref{app:proof}

The first theorem is a negative result that shows that under weak assumptions on $\Lambda^c$, one can construct a $\Lambda$ under which the naive padding scheme is almost surely guaranteed to not detect the errorful version of $A$ in $B$.
Indeed, it can be guaranteed to recover zero true vertex correspondences with high probability.

\begin{theorem}\label{thm:naive}
Suppose $R\in [0,1]^{n_c\times n_c}$ satisfies $R<\rho$ entry-wise for some scalar $\rho\in(0,1)$.
Suppose $n_c < n/2$, $\beta \in (0, 1)$, and entry-wise $\Lambda^c \leq\beta$ and $\Lambda^c = \omega(n_c^{-1}\log n_c)$.
Then there exists $\Lambda$ such that for $A,B \sim \CER(\Lambda, R)$, using the naive padding scheme, 
\begin{equation}
\label{eq:badperm}
\Pr[P^* \neq I] \geq 1 - \exp\left\{ -C \epsilon^2 (\log n_c)^2 \right\},
\end{equation}
for some universal constant $C>0$ and $\epsilon\in(0, 1-\beta-(1-\beta)\rho)$.
\end{theorem}
\noindent This result hinges on the fact that the naive padding scheme finds the best matching subgraph, rather than the best matching induced subgraph; indeed, there is no penalty for matching non-edges in $A$ to edges in $B$.
Hence, if $B$ has a dense substructure of size $n_c$ which does not correspond to $A$ then the naive padding scheme will match $A$ to that dense structure with high probability, regardless of the underlying correlation.
The proof demonstrates an extreme case where all vertices in $A$ are matched (incorrectly) to vertices in $B$ without true matches.
Deviations on the proof will yield that any number of vertices may be mismatched with only a mildly denser substructure in $B$.
Recall that since the centered padding scheme and the naive padding scheme are equivalent when $n=n_c$, the requirement that $n_c$ is sufficiently small is necessary.

On the other hand, the centered padding scheme can be guaranteed to result in the correct detection of the subgraph even when the number of vertices in $B$ is exponentially larger than the number of vertices in $A$ provided $R > 1/2 + \epsilon$.
\begin{theorem}\label{thm:center}
Suppose that $A,B\sim \CER(\Lambda, R)$ with $R_{uv} \in [1/2 + \epsilon, 1]$ and $\Lambda_{uv} \in [\alpha,1-\alpha]$
for $\alpha,\epsilon\in (0,1/2)$.
It holds that 
$\frac{\log(n)}{\epsilon^2 n_c \alpha (1 - \alpha)^2} = o(1)$
implies that using the centered padding scheme
$$\Pr[ P^* \neq I ]\leq 2\exp\left\{ -\Theta( n_c (\epsilon\alpha (1-\alpha))^2)\right\}.$$
\end{theorem}
\noindent Stated simply, this theorem only requires that the correlations are sufficiently large to guarantee that large subgraphs of logarithmic size can be found via an oracle GM algorithm.

Analogous to the naive padding scheme, the centered padding scheme does have its limitations.
The Theorem below indicates that if $n_c$ is too small then with high probability, the centered padding scheme will fail to recover the true correspondence for certain $\Lambda, R$.
Importantly, for the centered padding scheme this is a much weaker limitation than for the naive scheme.

\begin{theorem}
\label{thm:lowbnd2}
Suppose $R,\Lambda_c \in [\alpha,1-\alpha]^{n_c \times n_c}$ for some $\alpha\in (0,1/2)$.
There exists a constants $\xi, \epsilon>0$ such that if $n_c<\xi\log n$ then there exists a $\Lambda\in[0,1]^{n\times n}$ such that
$$\Lambda_{uv} (1-\Lambda_{uv}) R_{uv} \geq 1/2 + \epsilon$$ for all $u\neq v\in[n_c]$, and for $A,B\sim \CER(\Lambda, R)$, using the centered padding scheme,
\begin{align*}
\Pr\left[ P^* \neq I \right]=1-o(1).
\end{align*}
\end{theorem}
\noindent
Note, that the $O(\log n)$ rate corresponds to the information theoretic limit for detecting a planted clique within a random graph, where detecting a clique of size less than $2\log_2 n$ is impossible~\cite{alon1995color}.

Finally, while the oracle padding is inaccessible for general $\Lambda$, it represents the optimal padding scheme as it eliminates any empirical correlations introduced by $\Lambda$ leaving only the theoretical correlations from $R$.
\begin{theorem}\label{thm:oracle}
Suppose that $A,B\sim \CER(\Lambda, R)$ with $R_{uv} \in [\rho, 1]$ and $\Lambda_{uv} \in [\alpha, 1 - \alpha]$,
for some $\alpha, \rho \in (0, 1)$.
Then, using the centered padding scheme,
$$\Pr[ P^* \neq I ]\leq 2\exp\left\{ -\Theta( n_c (\rho\alpha (1-\alpha))^2)\right\}.$$
\end{theorem}

\begin{remark}
Note that while the padding schemes are still not equivalent in the homogeneous \ErdosRenyi{} case, the distribution of the matrix $P^*$ are equal under all three padding schemes and hence all are guaranteed to find the correct solution with high probability under identical scenarios.
\end{remark}

\subsection{Computation}\label{sec:computation}

Our approach to solve the graph matching problem in this setting will be identical to an approach for graphs of equal sizes; see the \texttt{FAQ} algorithm of~\cite{FAQ} for full detail.
In particular, we will relax the GMP constraints from $\mathcal{P}$ to $\mathcal{D}$ and use gradient descent starting at a given $D_0$. 
We will also incorporate seeds by appropriately modifying $D_0$.
This gradient ascent approach is then given by Algorithm~\ref{alg:faq}.
\begin{algorithm}[ht!]
\caption{Fast approximate quadratic assign program (FAQ) algorithm~\cite{FAQ} for graph matching.}\label{alg:faq}
\KwData{$A,B\in\AM$, $D_0\in \DSM$, $k=0$}
\While{not converged}{
\nl $P_k\leftarrow \argmin_{P\in \mathcal{P}} -\tr(\tA D_k \tB P)$\; \label{al:grad}
\nl $\alpha_k\leftarrow \argmin_{\alpha\in [0,1]} -\tr(\tA D_\alpha \tB D_\alpha)$, where $D_\alpha= \alpha D_k + (1-\alpha) P_k$\; \label{al:line}
\nl  $D_{k+1}\leftarrow D_{\alpha_k}$ and $k\leftarrow k+1$\; 
}
\nl Project  $D_{k}$ onto $\PM{}$ yielding $P^*$\; \label{al:project}
\nl Return $P^*$
\end{algorithm}

Note that using any of the padding schemes, we do not need to store or compute the entire matrices $D_k$ or $P_k$ as the objective function only depends on their first $n_c$ rows. 
Hence, lines \ref{al:grad} and \ref{al:project} can be simplified and accomplished by searching over the set of $n_c \times n$ matrices corresponding to injections from $[n_c]$ to $[n]$, or equivalently the first $n_c$ rows of permutation matrices.
In this way,  lines \ref{al:grad} and \ref{al:project}  can be solved effectively by variants of the Hungarian algorithm for non-square matrices~\cite{Munkres1957-ny}.
Line~\ref{al:line} is a quadratic equation in $\alpha$ and is easily solved.
Furthermore, by exploiting the fact that $\tilde{A}$ and $\tilde{B}$ are likely formed by the difference between a sparse and low-rank matrix, fast matrix multiplication and storage can also be exploited for further runtime improvements.

The computational complexity of each iteration is that of solving a rectangular linear assignment problem which can be solved in $O(n^2 n_c)$. 
The memory/storage complexity for the graphs is $n^2$ for dense matrices and on the order of the number of edges for sparse graphs.
For the doubly stochastic matrix, this can be stored in $O(n_cn)$ space. 
However, as $D_k = \left(\prod_{j=1}^k \alpha_j\right)D_0 +  \sum_{i=1}^k \left(\prod_{j=i}^k \alpha_j\right)  P_i$ memory (and computational) gains can be made as long as $D_0$ has a sparse plus low-rank structure.
As an example, below we consider $D_0 = \alpha_0 \1/n + (1-\alpha_0) P$ for some permutation matrix $P$.
In this case $D_0$ can be stored effectively with $O(n_c)$ memory and $D_k$ can be stored with at most $O(kn_c)$ memory.
The most memory intensive aspect of the problem is storing the gradient, which in general will not be sparse, however it may also have a sparse plus low-rank structure.

The convergence criterion is generally easy to check as the optimal doubly stochastic matrix is frequently itself a permutation matrix, which also means the final projection step can be omitted.
While this algorithm is not guaranteed to converge to a global optimum, if there are enough seeds or if the matrix $D_0$ is sufficiently close to the identity, the local maximum which this procedure converges to will likely be the identity.

\begin{figure}[t]
\begin{subfigure}[t]{\linewidth}
  \centering
  \includegraphics[width=.8\textwidth]{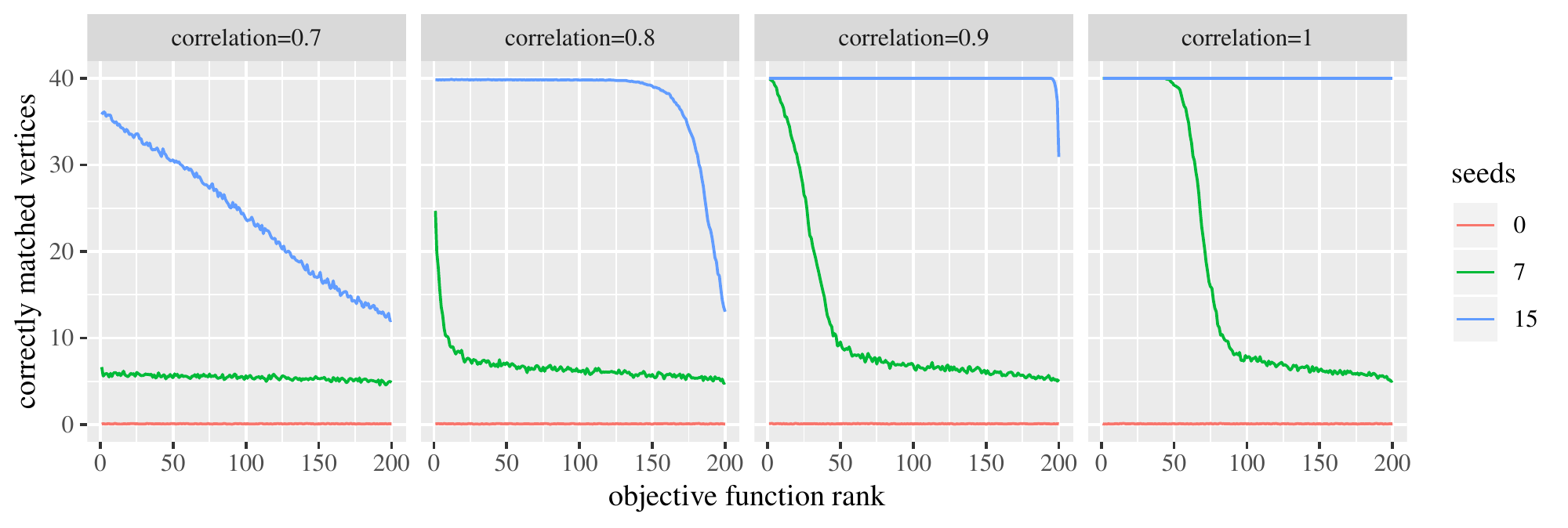}
  \caption{We plot the average number of matched vertices, out of 40, which were correctly matched according to the rank of the objective function.
  The different colors correspond to using either 0, 7, or 15 soft seeds and each panel corresponds to the homogeneous \ErdosRenyi{}, as described in Section~\ref{sec:cor_homog_er}.}
  \label{fig:acc_v_rank}
\end{subfigure}
\begin{subfigure}[t]{\linewidth}
  \centering
  \includegraphics[width=.8\textwidth]{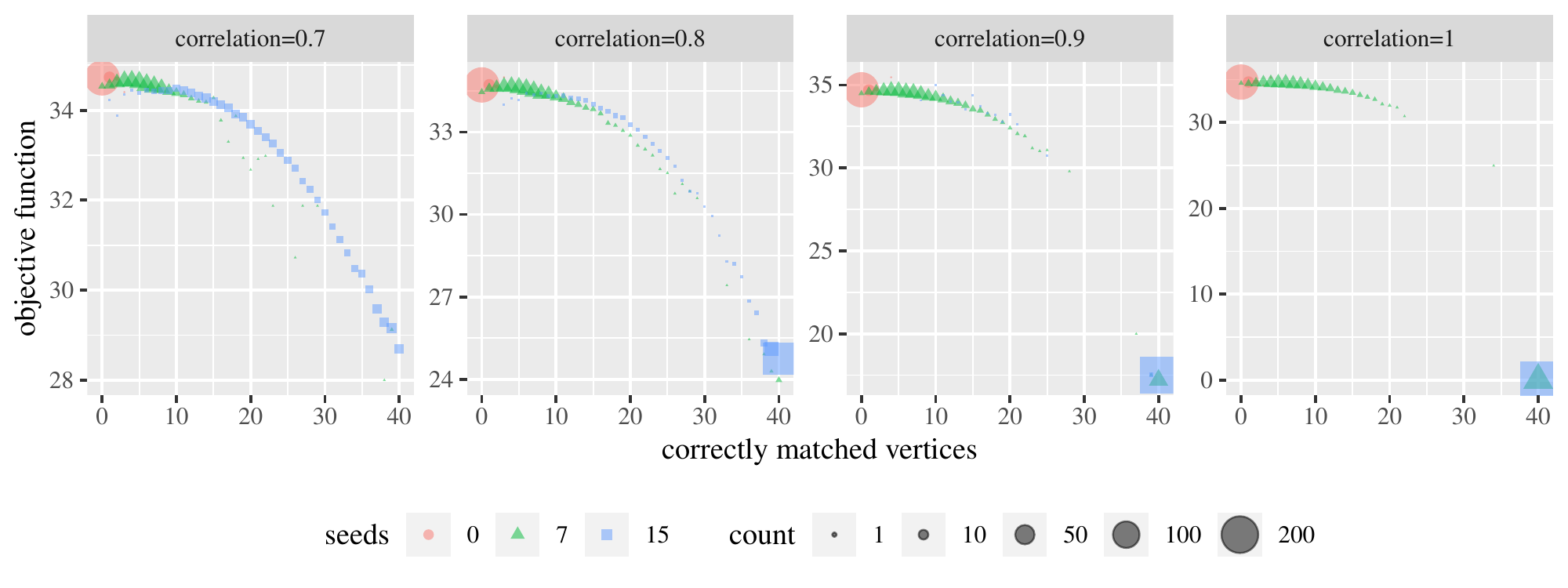}
  \caption{Each point denotes the conditional mean of the objective function given the number of matched vertices.
  Each point is scaled according the average number of restarts that achieve the given number of matched vertices.
  Note that for correlations $\geq 0.8$, there is a large gap in the expected objective function value for when there are $\approx 40$ correct matches as compared to $<40$. }
  \label{fig:f_v_acc}
\end{subfigure}
\caption{Results for homogeneous \ErdosRenyi{} experiments}
\label{fig:homoger}
\end{figure}

\section{Experiments}\label{sec:experiments}

In this section we demonstrate the effectiveness of a graph matching matched filter for errorful subgraph detection in both synthetic and real data settings.
Without loss of generality, we will assume that the seeds are the first $s$ vertices for some $s\leq n_c$.
The matched filter algorithm is given in Algorithm~\ref{alg:mf}
\begin{algorithm}[ht!]
\caption{Matched filters algorithm via random restarts, soft seeding, and gradient descent.}\label{alg:mf}
 \KwData{Template $A$, network $B$, number of seeds $s$, number of MC replicates $M$, padding scheme}
 \KwResult{Matchings $\{B_1,B_2,\cdots,B_M\}$}
 \nl $\tA, \tB \leftarrow$ padded and centered $A,B$ according to padding scheme\;
\For{$i\gets1$ \KwTo $M$ }{
\nl Generate a random doubly stochastic matrix $$D_0 \leftarrow I_s \oplus \left(\alpha P + \frac{1-\alpha}{n-s} \1_{n-s}\right),$$ where $P$ is a random permutation matrix and $\alpha\sim \mathrm{Unif}(0,1)$\label{al:rds}\;
\nl Initialize FAQ at $D_0$; match $\tA$ to $\tB$\;
\nl FAQ output is $P^*$, assignment of vertices in $B$ to vertices in $A$
}
\end{algorithm}

As described in Line~\ref{al:rds}, we initialize the FAQ algorithm at a random start point which initially aligns the seeded vertices across graphs.
In particular, our random starting point has the identity as its principal $s\times s$ submatrix, corresponding to the seeds.
The remaining matrix is sampled as a random convex combination of the the baricenter matrix $\frac{1}{n-s}\1_{n-s}$ and a random permutation matrix sampled uniformly from $\PM_{n-s}$.
Repeating this process $M$ times, we output $M$ potential matches for the subgraph $A$ in $B$.

To explore the effectiveness of this procedure, we consider applying our matched filter in the following scenarios: pairs of homogeneous $\CER$ graphs where we vary the uniform correlation, a planted partition model where the goal is to find the dense partition, a heterogeneous model with two different subgraph sampling schemes, and an application to {\it Drosophila} and Human connectomes.

\subsection[Correlated Homogeneous Erdos Renyi]{Correlated Homogeneous \ErdosRenyi{}}
\label{sec:cor_homog_er}

As a synthetic data example, we will consider subgraph detection in the $\CER{}(\lambda\1, \rho\1)$ with $\lambda=0.5$ (i.e., the maximum entropy ER model) for $\rho=0.7, 0.8, 0.9,$ and $1$.
When $\rho=1$, the induced subgraph of $B$ corresponding to $A$ is exactly isomorphic to $A$.
We consider $n=500$, $n_c=40$, and $M=200$ and we replicated the entire experiment for $200$ sampled graph pairs.
For each experiment, we rank the objective function values, computed as $\sum_{i,j < n_c} (\tA - P \tB P^T)^2_{ij}$, among the $M=200$ random restarts, from smallest to largest.

Figure~\ref{fig:acc_v_rank} shows the average performance, using the centered padding scheme, across the replicates at each rank.
As is evident, without seeds very few vertices are matched correctly, even when the corresponding graphs are isomorphic.
However, as both the number of seeds increase and as the correlations increase, the performance of the procedure becomes very good.
Note that when $\rho = 1$, we would expect exact subgraph matching procedures to perfectly recover the correspondence, while they would fail for all lower correlations.
These exact procedures do not require seeds but generally scale poorly with the number of vertices in each graph.

Interpreting the correlation $0.9$ panel, if $15$ seeds are used then nearly all restarts will yield a perfect recovery of the original correspondence.
Hence, only one restart is required.
Using 7 seeds, most restarts yield suboptimal performance, with only slightly more than the original 7 seeds recovered.
However, if we perform $200$ restarts and choose the best among them then this will often yield a near perfect recovery of the original correspondence.
Similar interpretations can be made for the other correlations.


Figure~\ref{fig:f_v_acc} gives a complementary view of this same simulation, showing the average value of the objective function as a function of the number of matched vertices.
Each point is scaled according to the average number of restarts (out of 200) which achieve the given number of matched vertices.
For $\rho=0.7$, the objective function decreases smoothly with the number of matched vertices, and the number of matched vertices is relatively evenly spread, especially for 15 seeds.
On the other hand, for higher correlations and for 15 seeds, there is a substantial gap in both the objective function and the matched vertices.

This gap is highly important for practical applications.
Indeed, the presence of a gap in the objective function would indicate that the correct matching was likely found.
Without the gap, the guarantees that the best match in terms of the objective function corresponds to a correct match is less certain (at least to practically ascertain).
Knowledge of the size of the gap would also allow for an adaptive algorithm which performs random restarts until a sufficient objective function gap is observed.
This is discussed further in Section~\ref{sec:dis}.

Finally, table one shows average run times, on a standard laptop, for each restart under each parameterization.
Even under the worst of the settings, 200 restarts can be performed in less than 10 minutes.
It is worth noting that better algorithmic performance correlates to faster algorithmic runtime, a feature which was observed in previous work using the \texttt{FAQ} algorithm~\cite{FAP}.

\begin{table}
\centering
\begin{tabular}{l|llll}
\toprule
    & \multicolumn{4}{c}{runtime in seconds} \\
seeds & $\rho=0.7$       & $\rho=0.8$       & $\rho=0.9$       & $\rho=1$ \\
\midrule
0     & $1.96 \pm 0.565$ & $1.92 \pm 0.478$ & $1.9\phantom{0} \pm 0.479$  & $1.91 \pm 0.48$\\
7     & $1.85 \pm 0.476$ & $1.8\phantom{0} \pm 0.466$  & $1.74 \pm 0.484$ & $1.58 \pm 0.546$\\
15    & $1.5\phantom{0} \pm 0.467$  & $1.03 \pm 0.395$ & $0.73 \pm 0.22$ & $0.61 \pm 0.164$\\
\bottomrule
\end{tabular}
\caption{Runtimes in seconds per restart for each setting in terms of the number of seeds and the correlation $\rho$ in the homogeneous \ErdosRenyi{} setting.}
\label{tab:runtime}
\end{table}

\subsection[Planted Partition Correlated Erdos-Renyi graphs]{Planted Partition Correlated \ErdosRenyi{} graphs}\label{sec:pp}

As a another scenario, we consider the problem where the larger graph has a denser subgraph which corresponds to the smaller graph.
We take $n=500$, $n_c=40$, $A,B\sim \CER(\Lambda, 0.9\1{})$ with $$\Lambda=\begin{pmatrix}
  q\1_{n_c} & p\1_{n-n_c,n}\\
  p\1_{n,n-n_c} & p\1_{n-n_c}
\end{pmatrix},$$
where $p=0.25$ and $q\in \{0.25, 0.3, \dotsc, 0.5\}$.
For 100 random restarts, 7 seeds and using the centered padding, we replicated each setting 200 times.
These parameters were chosen as illustrative examples that demonstrate the general phenomena described below.

Figure~\ref{fig:pc_acc} shows the average number of correctly matched vertices, using the centered scheme, at each objective function rank for each value of $q$.
As would be expected, the problem becomes significantly easier as $q$ increases.
A similar gap phenomenon is observed in this setting, as long as $q$ is sufficiently large, though we do not illustrate it here.

\begin{figure}[tb]
  \centering
  \includegraphics[width=.7\linewidth]{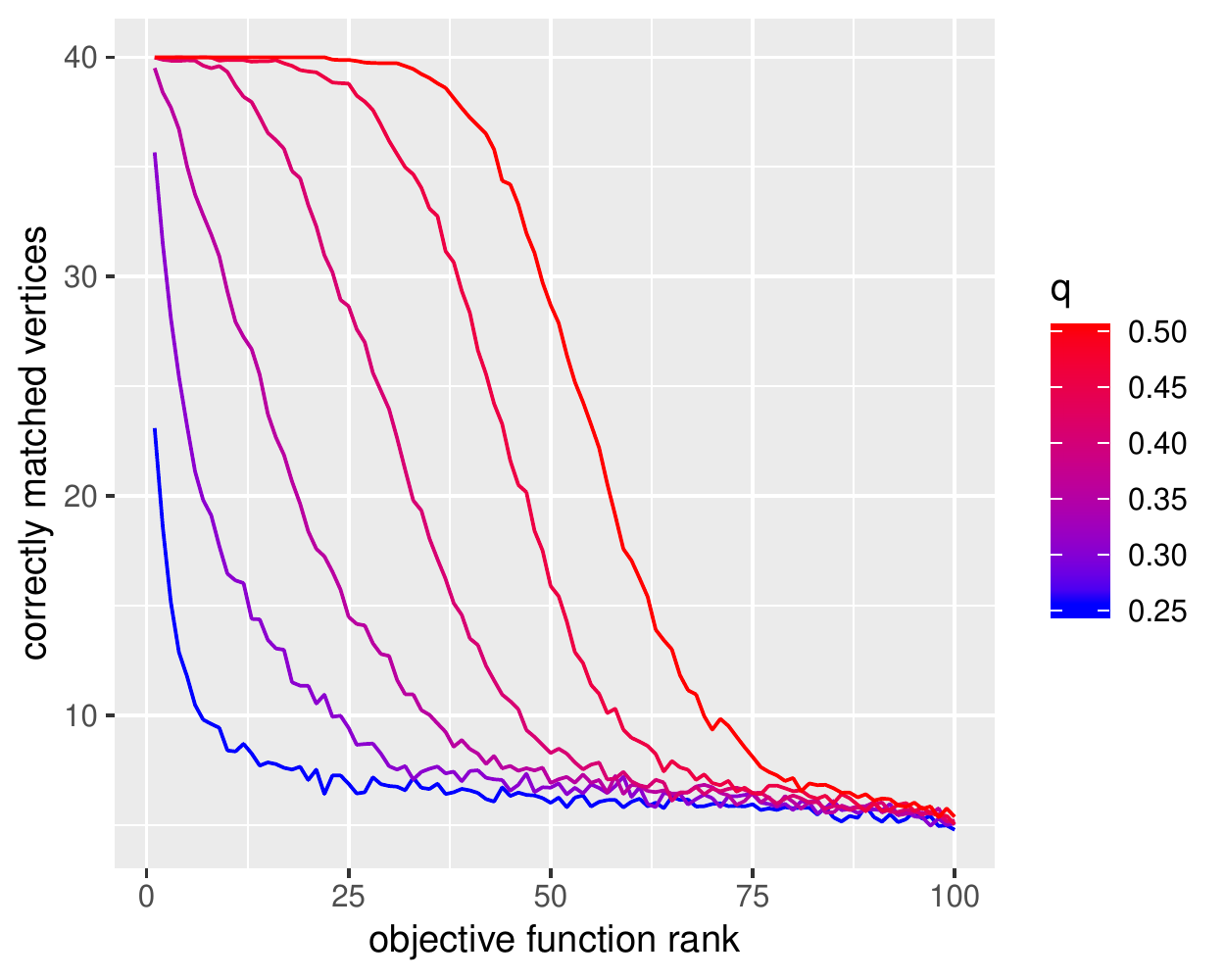}
  \caption{Number of correctly matched vertices as a function of the objective function rank for the planted partition model described in Section~\ref{sec:pp}.
  Each line corresponds to the density of the induced subgraph, ranging from $q=0.25$ to $0.5$.}
  \label{fig:pc_acc}
\end{figure}

\subsection{Random Dot Product Graph}\label{sec:rdpg}

In order to better evaluate the performance of the procedure under heterogeneous graph models, we will now consider graphs from the random dot product graph (RDPG) model~\cite{young2007random,nickel2006random}.
The RDPG model provides a generalization of the stochastic blockmodel, including mixed-membership~\cite{Airoldi2008-hf and degree-correction~\cite{karrer2011stochastic}, allowing for multiple forms heterogeneity in vertex behavior.
As we show below, this heterogeneity allows for the exploration of how the latent structure of the induced subgraph influences performance for our methods.
For this model, the random adjacency matrix $A$ has a low-rank expectation $\Ex[A]$ which is conducive to approximation using spectral methods.
Below we center by a low-rank approximation of the adjacency matrix, which is probabilistically guaranteed to be an accurate estimate of $\Ex[A]$~\cite{Athreya2017-ee}.
This serves to approximate the oracle procedure.}

Each vertex $v$ is associated with a latent position $X_v\in \Re^2$,
which is sampled uniformly from the set $\{(x,y)\in (0,1)^2: x+y < 1\}$.
In this case the matrix $\Lambda$ is given by $\Lambda_{uv} = X_u^T X_v$.
We again take $n=500$ and $n_c=40$ and sample $A,B\sim \CER{}(\Lambda, 0.9\1)$.
The matched subgraphs are sampled in two different ways.
Under one sampling design, the matched nodes simply correspond to a random sample of the latent positions. 
Under the ``max-angle'' design, the matched nodes are selected to be those with the largest ratio between the second and first elements of the latent positions.
Figure~\ref{fig:rdpg_lp} shows a representative set of latent positions corresponding to this model.
The red circles correspond to the max-angle sampling and the points in blue correspond to a possible random sampling (these do not overlap for the purposes of visualization).

\begin{figure}[tb]
  \centering
  \includegraphics[width=.7\linewidth]{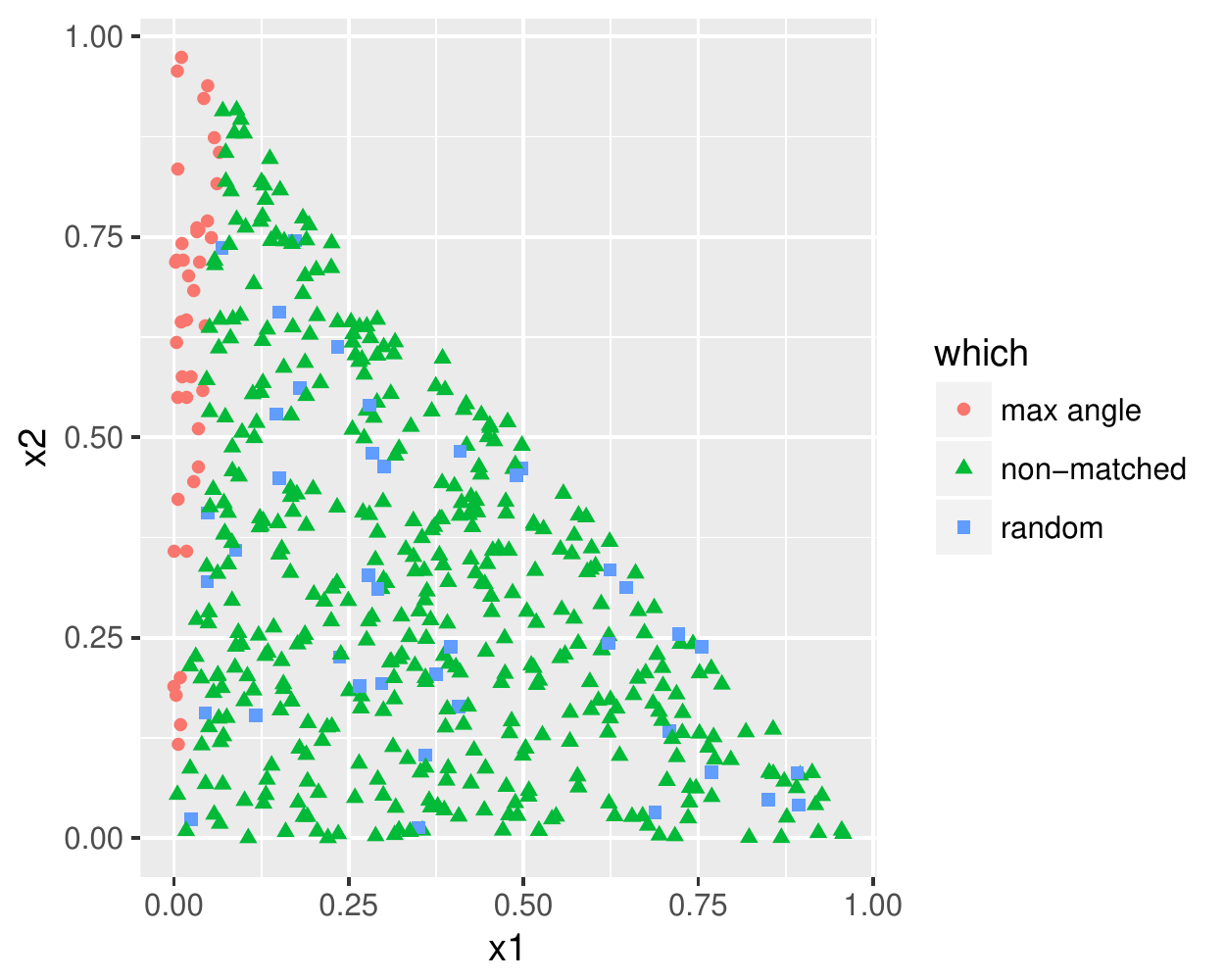}
  \caption{A representative set of latent positions for the random dot product graph model.
  The red circles correspond to the max-angle sampling and the points in blue correspond to a possible random sampling (these do not overlap for the purposes of visualization).
}
  \label{fig:rdpg_lp}
\end{figure}

We sampled 200 distinct graphs under each subgraph sampling setting and performed 100 random restarts for each sampled graphs.
For each restart, the matchings were performed with the naive, centered, and two ``oracle''-type padding schemes. 
As the model above is a low-rank model, we centered each matrix according to its best rank-1 and rank-2 approximations to approximate the oracle padding scheme. 
Letting, $$\hat{\Lambda^{(M)}_r=\argmin_{\Lambda: \mathrm{rank}(\Lambda)=r} \|M-\Lambda\|_F^2}
,$$
for each of $M=A$ and $B$,
these schemes match $(A-\hat{\Lambda}^{(A)}_r)\oplus \0_{n-n_c}$ to $B-\hat{\Lambda}^{(B)}_r$ for $r=1,2$.

Figure~\ref{fig:rdpg_acc} shows the average accuracy at each rank of the objective function.
First, the max-angle subgraph is clearly much easier to discover than the random subgraph.
The center scheme had by far the best performance for the max-angle case and in the random case it also performed well.
The rank-2 centering scheme also performed well in the max-angle case and had very similar performance to the center scheme.
Both the rank-1 and naive padding schemes failed to display good performance in either case.

\begin{figure}[tb]
  \begin{subfigure}[t]{.485\linewidth}
  \centering
    \includegraphics[width = \linewidth]{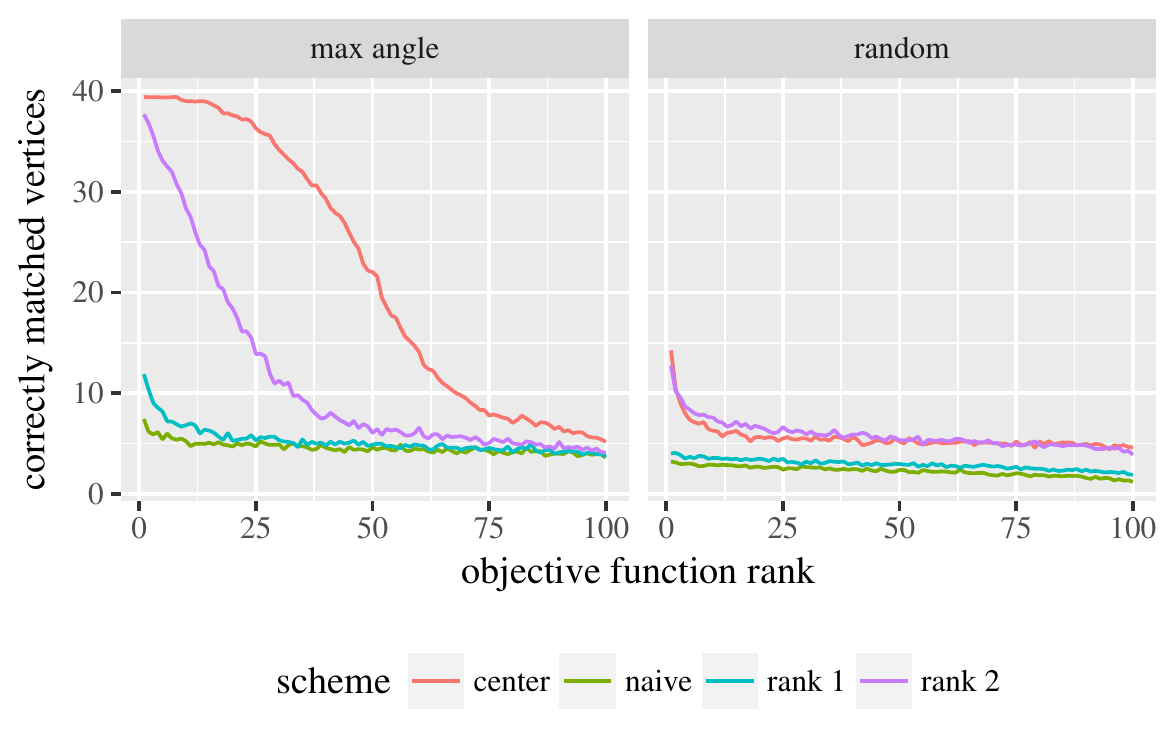}
    \caption{The number of correctly matched vertices out of 40 according the ranks of the objective function across 100 random restarts.
    The left panel corresponds to the scenario where the smaller graph is correlated with a subgraph determined by underlying vertex attributes and the right panel corresponds to a random subgraph.}
    \label{fig:rdpg_acc}
  \end{subfigure}\hfill
  \begin{subfigure}[t]{.485\linewidth}
  \centering
    \includegraphics[width = \linewidth]{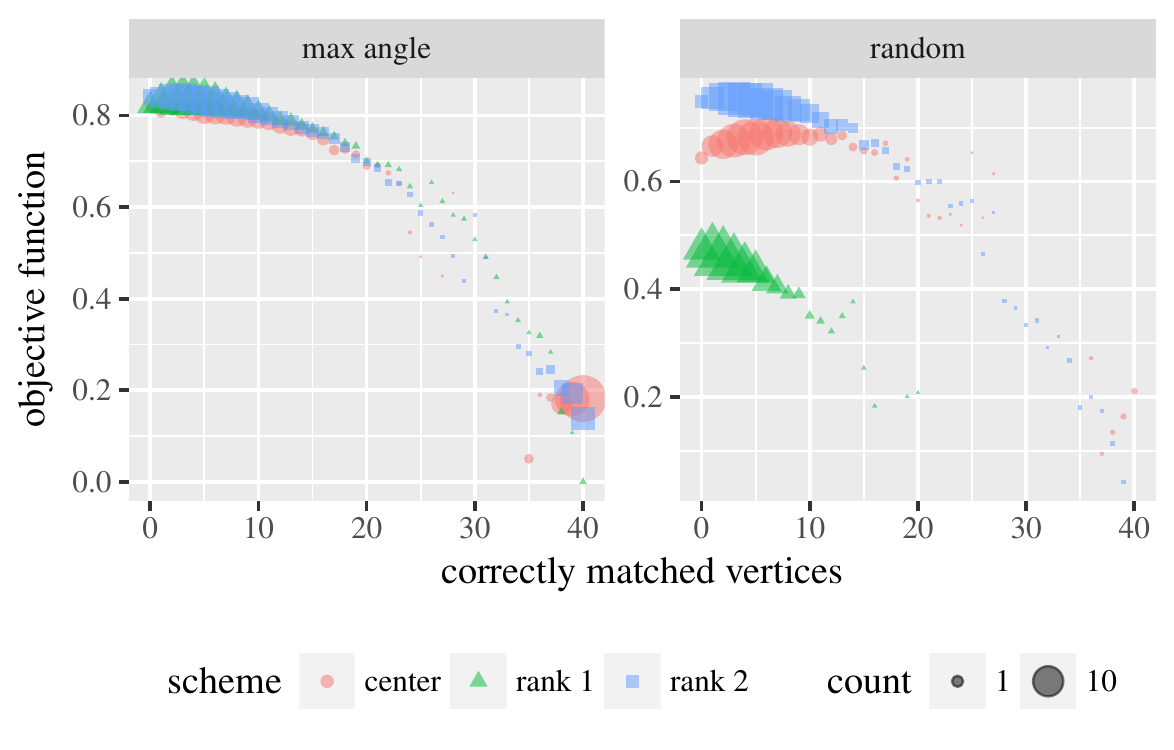}
  \caption{
  Each point denotes the conditional mean of the objective function given the number of matched vertices.
  Each point is scaled according the average number of restarts that achieve the given number of matched vertices.
  Note that despite the fact that there are fewer high accuracy solutions for the random subgraph, in both situations there is a substantial gap between the objective function value for high accuracy solutions and those with low accuracy.}
  \label{fig:rdpg_gap}
  \end{subfigure}
  \caption{Random dot product graph matching results.}
  \label{fig:rdpg}
\end{figure}

Under either scenario, the rank-one approximation eliminates most degree based differences between vertices.
In some cases this may be helpful if the degree information is not informative about the match, but in this case it appeared to hurt performance in both cases.
Rank-two centering in this case seems to restore some of the performance losses from rank-1 centering.
This also is somewhat unexpected since the heterogeneity from both principal directions can be used to help identify the vertices.
The reasons that the rank-2 centering is better than the rank-1 centering are unclear and an area of current investigation.
Note that constant centering maintains all differential vertex behavior which is a strong indicator of which vertices to match to.
The naive padding fails here as there are always sets of vertices with similar relative propensities for adjacencies to other vertices but with strictly larger absolute probabilities of adjacencies.

\begin{figure*}[t]
  \begin{subfigure}[t]{.44\textwidth}\vskip 0pt
  \includegraphics[width=\textwidth]{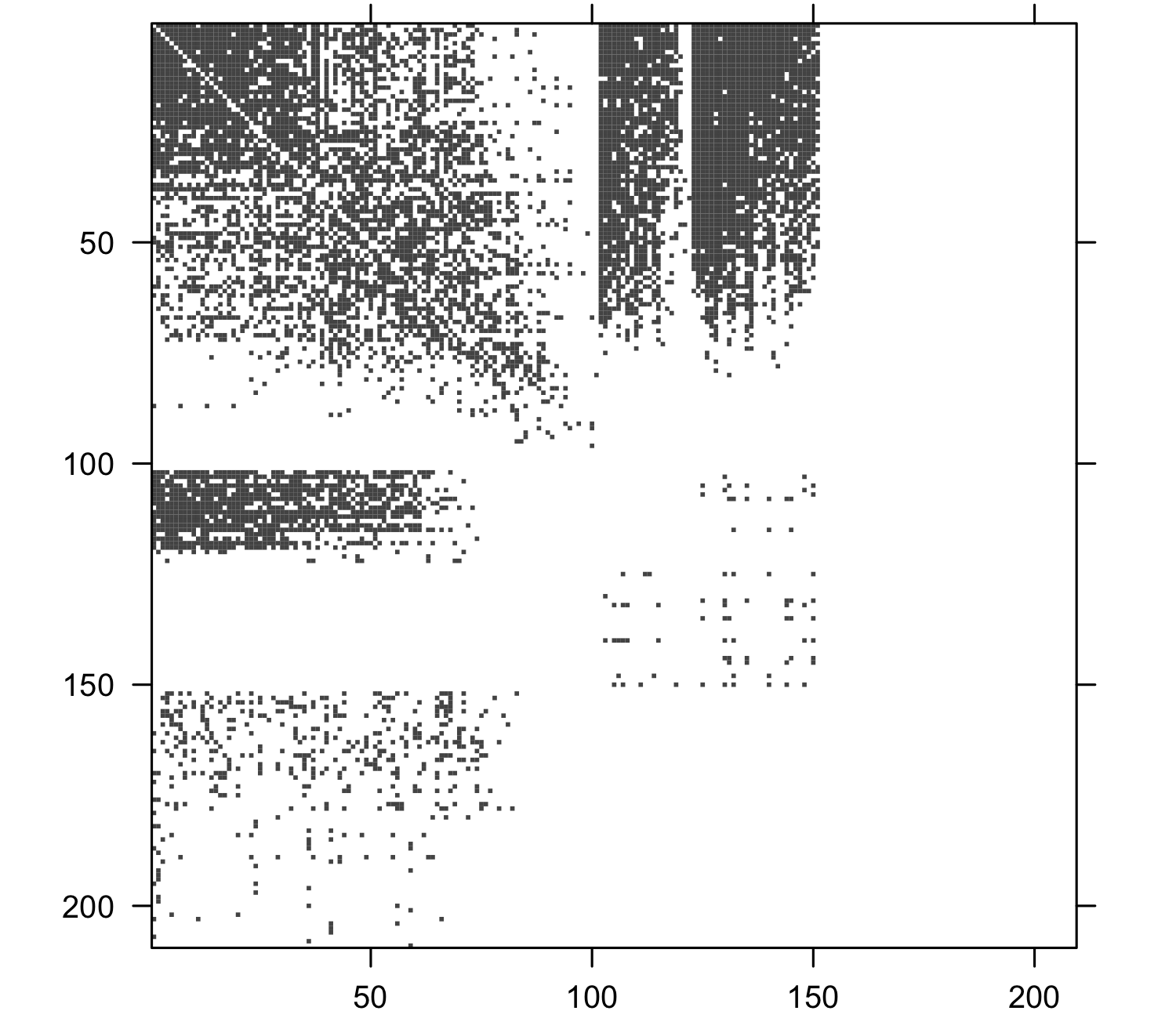}
  \caption{{\em Drosophila} left hemisphere}
  \label{fig:kc_left}
  \end{subfigure}
  \begin{subfigure}[t]{.29\textwidth}\vskip 0pt
  \includegraphics[width=0.862\textwidth]{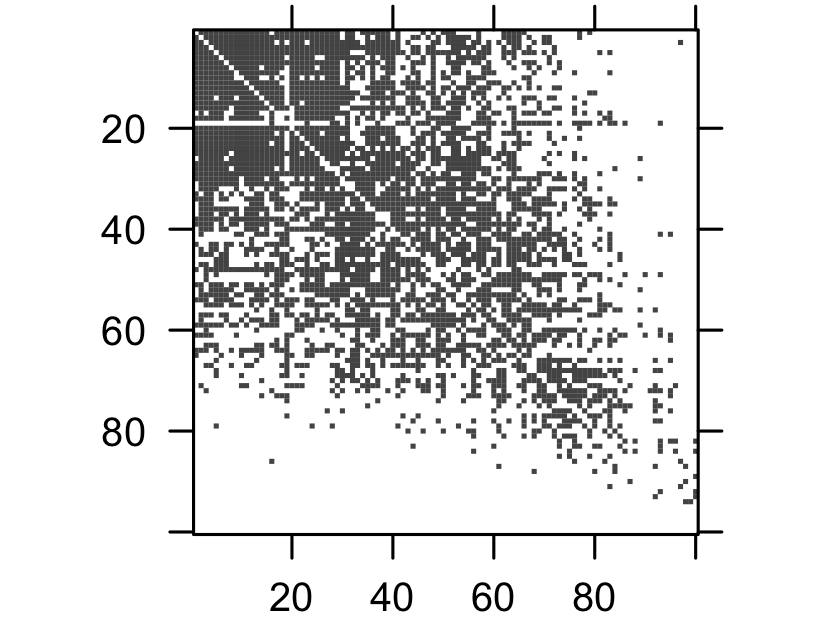}
  \caption{right Kenyon cells}
  \label{fig:kc_right}
  \end{subfigure}\hfill
  \begin{subfigure}[t]{.26\textwidth}\vskip 0pt
  \includegraphics[width=\textwidth]{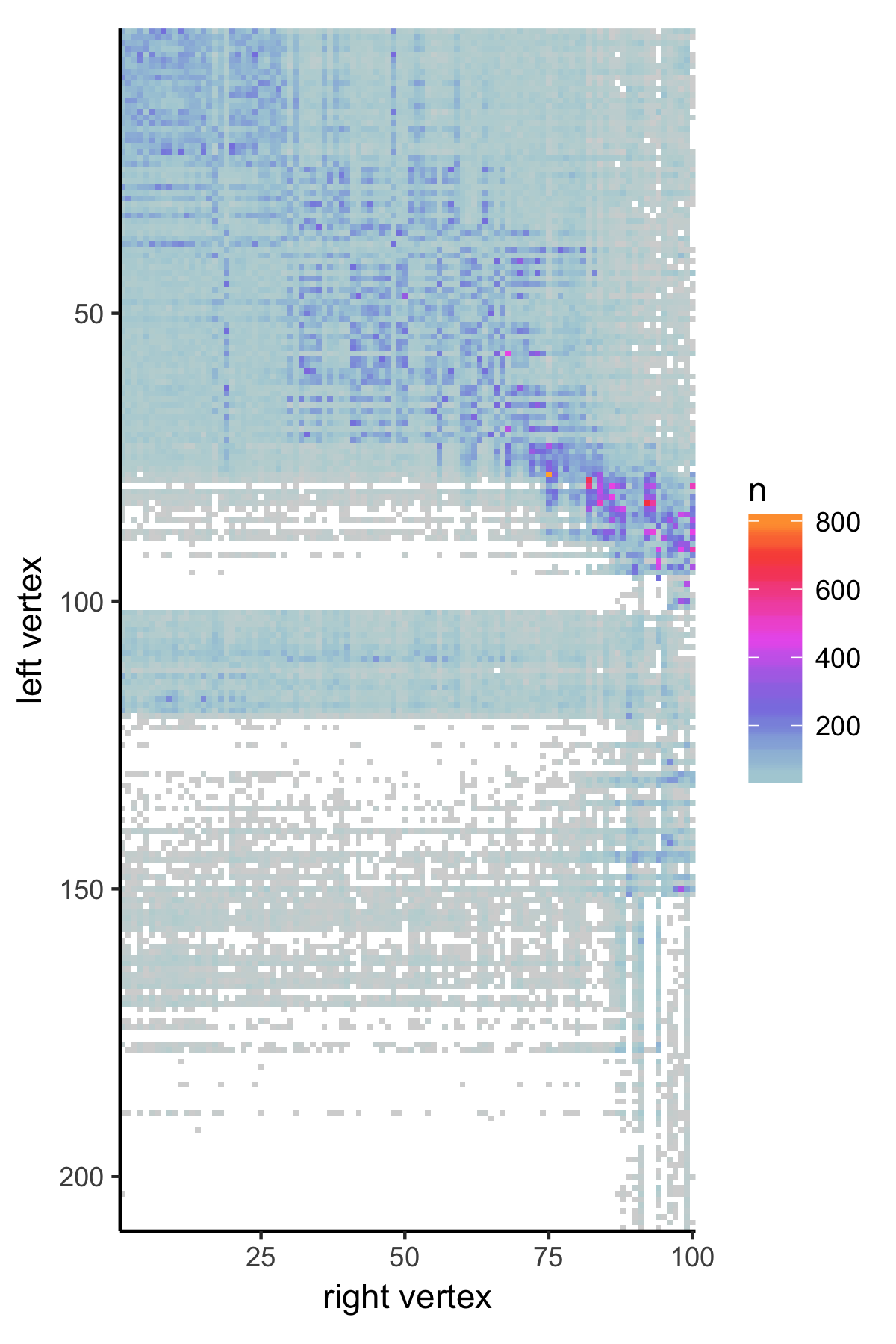}
  \caption{matched pair frequencies}
  \label{fig:kc_matched_pair}
  \end{subfigure}
  \caption{Connectomes and matching analysis of the {\em Drosophila} connectomes.
  (a) and (b) show adjacency matrices for the two {\it Drosophila} connectomes. Dark pixels correspond to a directed edge being present.
  The first 100 rows/columns of the left hemisphere are the K-cells.
  Panel (c) shows the number of times that each node in the right hemisphere was matched to each node in the left hemisphere.}
  \label{fig:kc_spy}
\end{figure*}

To understand why the max-angle problem is so much easier, note that, under either subgraph sampling scheme, there will be many nodes which have similar behavior to one or more of the sampled nodes due to having nearby latent positions.
Under random sampling, most of these similar nodes are likely to correspond to nodes that do not have a match and hence will introduce errors.
However, under max-angle sampling the similar nodes have a better chance of being other matchable nodes which will introduce fewer errors as the theoretical correlation will strongly distinguish these vertices.

Figure~\ref{fig:rdpg_gap} shows the average value of the objective function as a function of the number of matched vertices.
Since the different centering schemes have different objective functions which will have slightly different scales, we rescaled objectives functions for each scheme so that the largest value for each was set to zero.
The original objective functions for the rank-1, rank-2, and center schemes ranged from 7--23, 8--18, and 11--36 for the max angle setting and 15--25, 9--17, and 9--31 for the random setting, respectively.
Each point is sized according to the average number of restarts (out of 100) which achieve the given number of matched vertices.
Note that we again observe a gap in the objective function between the bulk of solutions with less than 20 correct matches and those with more than 35 correct matches.

\subsection{Finding Kenyon cells in a {\it Drosophila} connectome}

As an application to brain networks, we consider using the matched filter to locate an induced subgraph of Kenyon (K) cells  in the fully reconstructed {\it Drosophila} mushroom body of~\cite{eichler2017complete}.
Using the induced subgraph of the K-cells in the right hemisphere of the mushroom body (i.e., as $A$), we seek to find the K-cells on the left hemisphere (i.e., as $B$).
Although in this example, the K-cells are identified across both hemispheres, this was achieved only with great effort and expenditure.
Being able to use one hemisphere to locate structure in the other hemisphere could potentially allow for faster, cheaper neuron identification in future connectomes.
In particular, this can be extended to finding similar structures across connectomes collected from multiple individuals.

After initial data preprocessing, there are $n_c=100$ K-cells in each hemisphere and $n=209$ vertices total in the right hemisphere.
The pair of graphs are illustrated in Figure~\ref{fig:kc_spy}, with the first 100 nodes in the left image being the K-cells of the left hemisphere.

As the true correspondence between these pair of graphs is unknown, the accuracy is measured by the number of cells in the right K graph which are matched to K-cells in the left graph, rather then the number of correctly matched pairs.

We consider $s=0$ and $10$ soft seeds and $M=50,000$ random restarts for our matched filter.
The seeds were determined using cells which were given the same name across the two hemispheres.

In addition to performing the matching in our previous examples using naive, centered, and low-rank centered schemes, we also consider a second ``re-matching''.
Letting $D^*_1$ denote the doubly stochastic solution found by the first matching (prior to projecting to the permutations), we consider rematching starting at $D^*_1$ using a (different) low-rank centering.
In this way the initial matching may find gross similarities which are then refined by removing more empirical correlation between the graphs.
Whether this procedure is effective will of course depend on the particular choices for the first and second matching.

\begin{figure*}[tb]
\centering
\includegraphics[width=\textwidth]{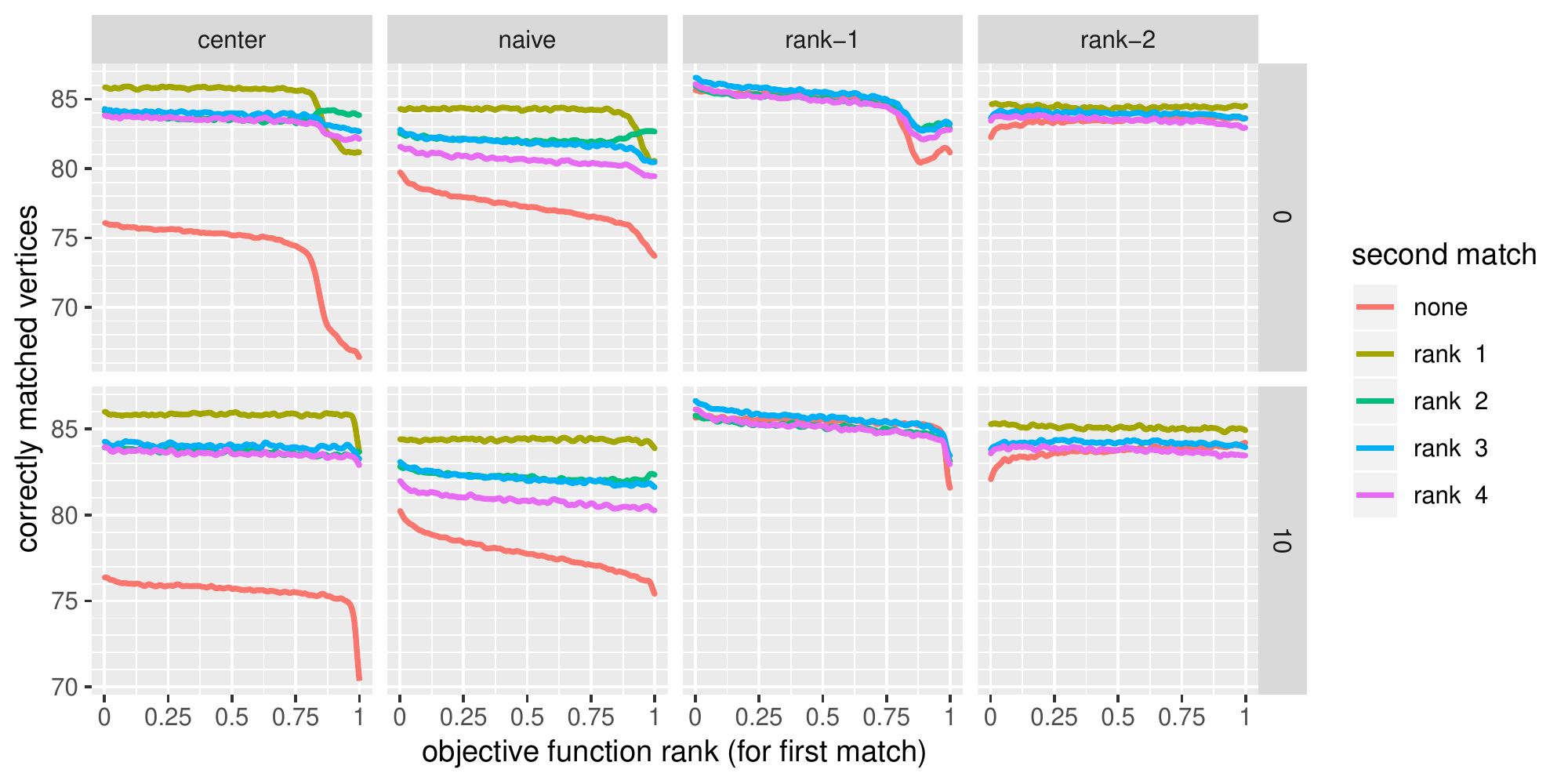}
\caption{These panels show the accuracy for detecting the 100 K-cells in the left hemisphere using the K-cells subgraph for the right hemisphere.
The four columns correspond to the 4 different centering options and each curve corresponds to the centering that was used for the second matching.
Kernel smoothing with bandwidth $0.02$ was used to smooth the average accuracy across the normalized objective function ranks (for the first match).
The two rows correspond to using 0 and 10 seeds.}
\label{fig:kc}
\end{figure*}

Figure~\ref{fig:kc} shows the number of correctly matched vertices at each rank of the objective function, after the first match. 
The ranks were linearly rescaled to be between 0 and 1.
Importantly, only the ranking of the initial ranking is used, (not of the second ranking).

In this instance we smoothed the performance across ranks using a Gaussian kernel smoother with bandwidth $0.02$.
The performance of each matching scheme and the scheme used for the second matching can all have substantial impacts on the overall performance.
For example, while the centered scheme alone has very poor performance, if it is followed by the rank-1 scheme then the performance becomes better than the naive.

The overall best performance, according to this smoothed performances estimate, is achieved by using the rank-1 scheme followed by the rank-3 scheme.
As these graphs have many matches with similar performance, each of which are local minima for the various objective functions, we fail to observe the substantial gap between accurate and inaccurate solutions. 
Indeed, most solutions under the rank-1 scheme achieve between 80 and 91 correct matches.

We note that in this instance the number of seeds barely impacts performance.
This is likely due to two reasons.
These seeds are merely guesses at some true correspondence and the best matchings often does not maintain all of the original seeds.
Indeed, using 10 seeds, nearly every match does not include at least one of the original seed pairs.
Also, as can be seen in Fig.~\ref{fig:kc} the graphs have a very distinctive structure where certain nodes, such as those with indices greater than 125, are very unlikely to be matched because of the directed nature of the graph.
These nodes can be excluded, which is the opposite side of the coin of having seeds, and also makes the matching problem easier.

In order to visualize matches across restarts, for each pair of vertices across graphs, we can count the number of restarts where the pair was matched.
Fig.~\ref{fig:kc_matched_pair} shows a heat map for these counts for the matches achieved using the rank-1 centering followed by the rank-3 centering using 0 seeds.
The total number of restarts was 5000, corresponding to the column sums of the corresponding matrix.}

Most of the matchings occur between Kenyon cells and even within these matching there are clear patterns of certain vertices frequently being matched to each other and groups of vertices that appear to be mapped together.
Note also that many match pairs occur infrequently, with 32\% of pairs never being matched.
58\% of pairs which were matched were matched less than 100 times.
This visualization and the associated matrix could easily be further analyzed to discover distinct patterns and grouping of vertices.

Table~\ref{tab:best_obj} shows the matching performance (out of 100) for the solution which achieved the smallest objective function (after the first matching).
Using this simple method for selecting a match, we achieve 89 correct matches.
The best overall matches were able to find 94 K-cells but these matches did not appear among the top 1000 objective functions values.

\begin{table}[tb]
    \caption{The number of K-cells in the left hemishpere detected by the match with the best objective function after the first match. The best among these, highlighted in bold, is 89 which was achieved by first performing a rank-1 centered match and then performing a rank-4 centered match, or a rank-2 centered match in the case of zero seeds.}
    \label{tab:best_obj}
    \centering

\begin{tabular}{rl|ccccc}
\toprule
  & & \multicolumn{5}{c}{\textbf{second match}} \\
\textbf{seeds} & \textbf{scheme} & none & rank-1 & rank-2 & rank-3 & rank-4\\
\midrule
0 & rank-1 & 86 &   & \textbf{89} & 88 & \textbf{89}\\
  & rank-2 & 82 & 82 &   & 83 & 85\\
  & center & 77 & 86 & 82 & 80 & 85\\
  & naive & 84 & 87 & 83 & 86 & 86\\
\midrule
10 & rank-1 & 84 &   & 85 & 86 & \textbf{89}\\
   & rank-2 & 76 & 84 &   & 84 & 81\\
   & center & 78 & 86 & 82 & 83 & 83\\
   & naive & 84 & 88 & 85 & 83 & 86\\
\bottomrule
\end{tabular}
\end{table}

One might imagine that better performance could be achieved by sorting according to the objective function value for the second match. 
However, this actually results in adverse behavior where lower objective functions result in poorer performance.
While we do not exactly understand this phenomenon, if we view the second match as only a refinement, then we postulate that lower objective function values have strayed farther from the initial match, which may reduce performance if the initial match is good.

\begin{table*}[tb]
    \caption{Performance comparisons for the large scale DTMRI examples.}
    \label{tab:dt_stat}
    \centering
{\tiny
\begin{tabular}{r|cc|ccc|ccc}\toprule
\multicolumn{3}{c|}{} & \multicolumn{6}{c}{\textbf{best matches}} \\
\textbf{region} & \textbf{\# nodes} & \textbf{\# edges} & \multicolumn{3}{c|}{\textbf{fewest edge errors}} & \multicolumn{3}{c}{\textbf{most correct vertices}} \\\midrule
left         & 20,412   & 922,793            & \textbf{\# correct} & \textbf{\# edge} & \textbf{\# seeds} & \textbf{\# correct} & \textbf{\# edge} & \textbf{\# seeds}  \\
right       & 20,401   & 921,865            & \textbf{matches}    & \textbf{errors}  &  & \textbf{matches}    & \textbf{errors} &  \\ \midrule
left \phantom{0}7        & 186      & 611                & 9          & 49      & 0 & 56         & 121 & 20 \\
right \phantom{0}7       & 173      & 646                & 0          & 81      & 0 & 65         & 576 & 20 \\
left 20       & 136      & 816                & 1          & 121     & 0 & 78         & 324 & 10 \\
right 20      & 192      & 1077               & 2          & 144     & 0 & 56         & 256 & 20 \\
left 33       & 70       & 369                & 0          & 36      & 0 & 45         & 289 & 20 \\
right 33      & 81       & 355                & 0          & 49      & 0 & 30         & 121 & 10 \\
left 34       & 206      & 1528               & 0          & 324     & 0 & 78         & 729 & 10 \\
right 34      & 173      & 1363               & 0          & 256     & 0 & 95         & 625 & 10 \\\bottomrule
\end{tabular}}
\end{table*}

\subsection[Large-scale DTMRI Graphs]{Large-scale DTMRI Graphs}

Diffusion tensor magnetic resonance imaging (DTMRI) is a form of MRI which allows for the estimation of coarse scale approximations of the networks of a human brain~\cite{Kiar2018-ti,Kiar2018-fl}, so called connectomes.
While the estimated graphs frequently only have 10s or 100s of vertices, recent tractography methods and parcellation techniques have allowed the creation of connectomes with tens of thousands of vertices.
To illustrate the ability of algorithm to scale to larger problems, we considered a matching problem on one such graph.

In particular, we used the DS72784 atlas from the neurodata repository~\cite{Kiar2018-ti}, and considered a problem analogous to finding Kenyon cells in the Drosophila connectome.
The graph has $\approx 40$ thousand nodes. 
Using the coarser Desikan atlas, each vertex can be assigned to one of 70 regions, 35 in each hemisphere.
Corresponding regions are numbered by $r$ and $r+35$, for $r=1,\dotsc,35$.

As in the Drosophila example, we considered the problem of matching the subgraph corresponding to a region in one hemisphere, to the other hemisphere.
In particular, we considered the matching of induced subgraphs of regions 7, 20, 33 and 34 in the left hemisphere to the right hemisphere, and matching regions 42, 55, 68, and 69, in the right hemisphere to the left hemisphere.
These regions were picked because they were relatively close in number of vertices and edges and not too large.
Details, for the numbers of vertices and edges in each graph are in the first three columns Table~\ref{tab:dt_stat}.

For these experiments we used the same randomized initialization procedure and considered using 0, 10 and 20 seeds.
As the true vertex correspondence is unknown, for seeds we randomly selected pairs of vertices from the matched regions.
Each of the eight region-to-hemisphere matchings was repeated 400 times for each seeding level.
While the number of iterations required per repetition varied, the time per iteration ranged from $5.6\pm 1.5$ seconds for the smallest graphs to $8.1\pm 2.6$ seconds for the largest graphs.

As the goal for these experiments was primarily to show the ability of our methods to scale to larger problems, we only considered using the centered padding scheme.
The results in terms of the number of vertices correctly identified as belonging to the corresponding region indicate that using this procedure we were unable to accurately locate the corresponding region.

The rightmost 6 columns of Table~\ref{tab:dt_stat} show the number of correctly matched vertices and the number of edge errors for the best edge error match, with the fewest number of edge errors, and the best region match, the match with the most vertices matched to the correct region.
We also indicate the number of seeds used for each of these matchings.
Note that the fewest edge error matching all used 0 seeds, had less than half the number of edge errors as the best region match, and generally did not find any vertices in the correct region.

Figure~\ref{fig:dt_example} shows an example of the two matchings for matching region 7 in the left hemisphere. 
The matching on the left is the best matching achieved in terms of minimizing the objective function and the matching on the right is the best matching achieved in terms of vertices correctly identified as belonging to the corresponding region.
The central figure shows the adjacency matrix for the original small graph and the top left and right show the adjacency matrices for the induced subgraphs that were found by the procedure.
The bottom left and right show the difference between the original graph and the matched graph, with missing edges shown in blue and extra edges shown in red.
Both matches appear to share the overall structure of the center plot but the left match matches many finer details that the right match does not.

\begin{figure*}[tb]
	\centering
	\begin{subfigure}[t]{.3\textwidth}\vskip 0pt
	\includegraphics[width=\linewidth]{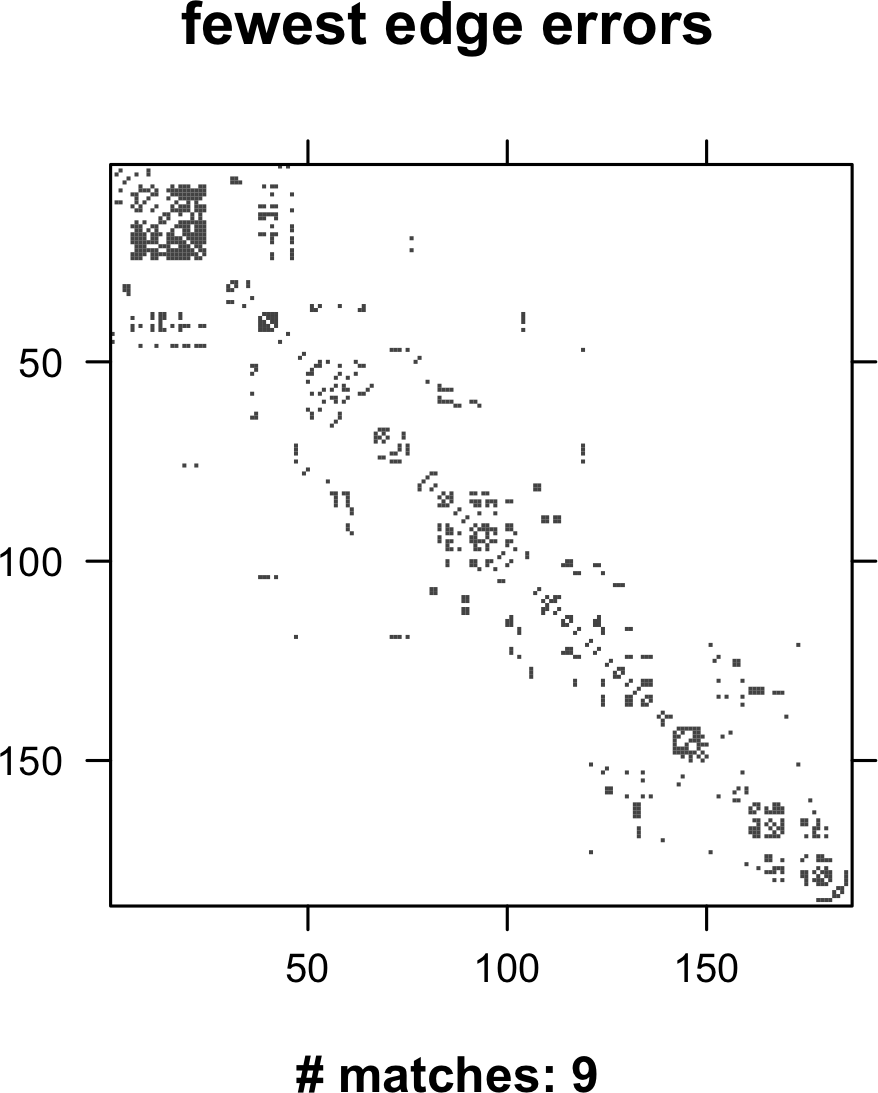}\\[3pt]
	\includegraphics[width=\linewidth]{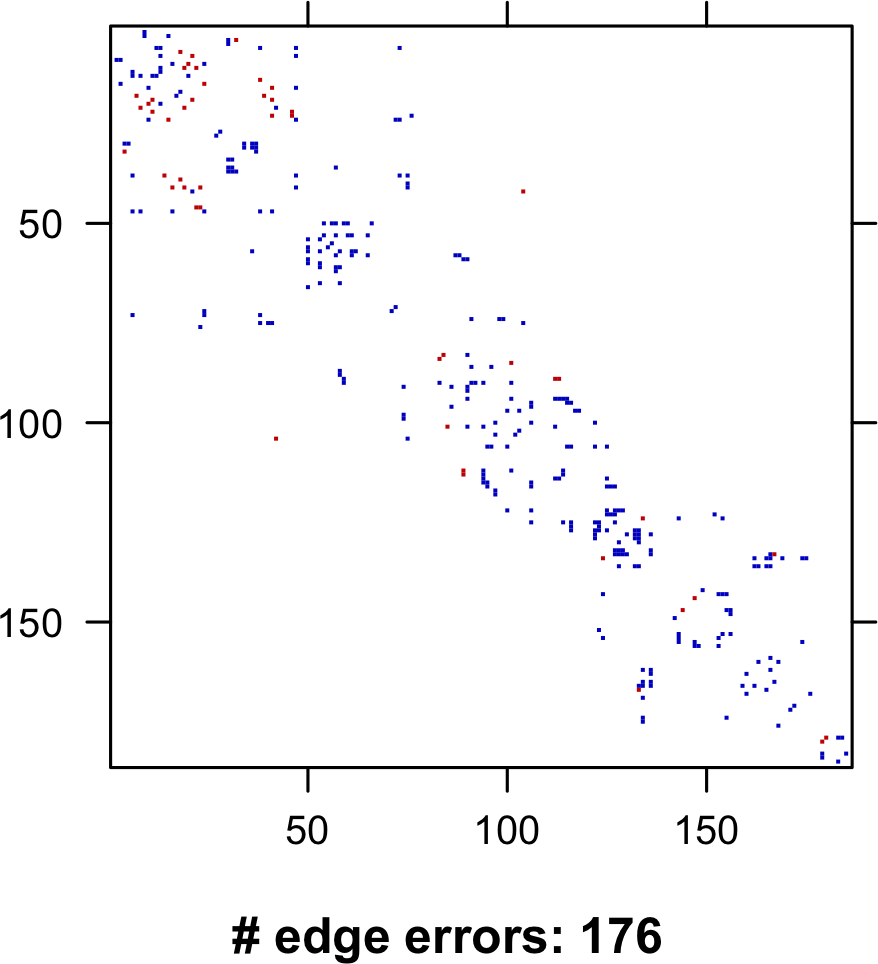}
	\caption{match with fewest edge errors}
	\end{subfigure}\hfill
	\begin{subfigure}[t]{.3\textwidth}\vskip 0pt
	\includegraphics[width=\linewidth]{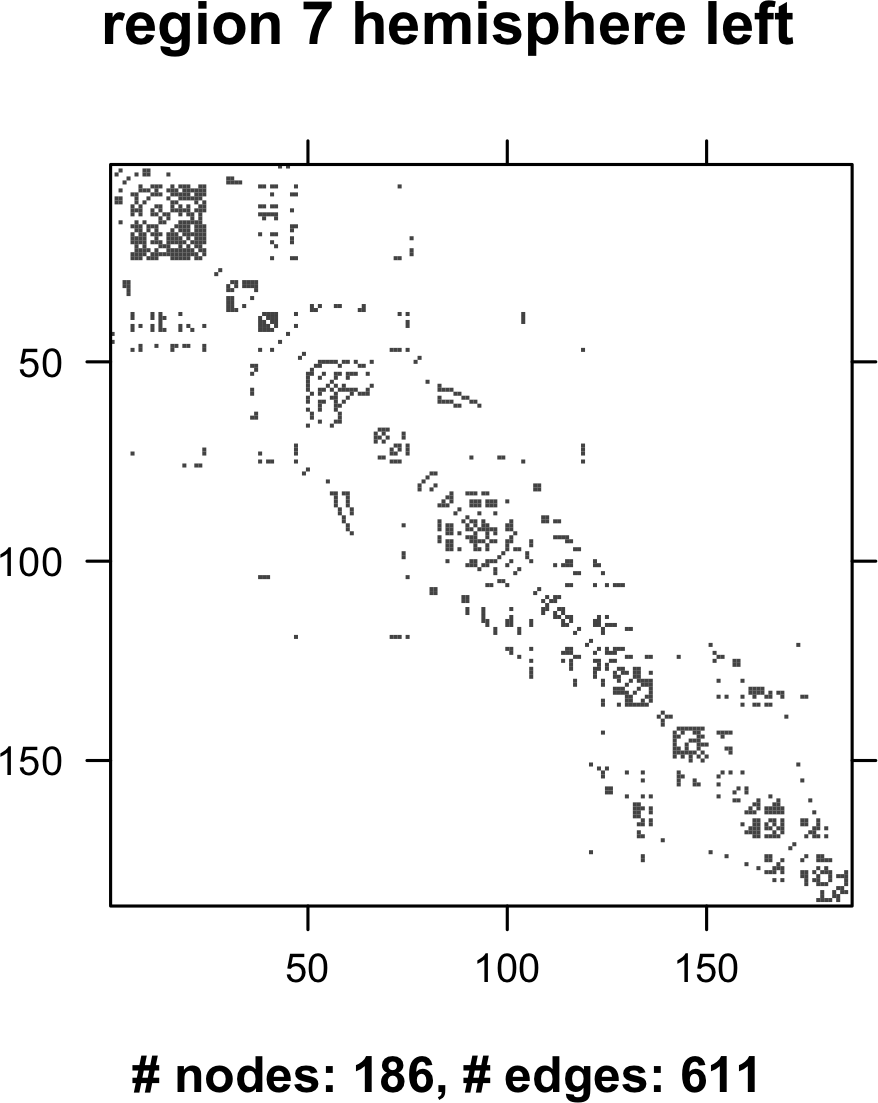}
	\caption{original graph
	}
	\end{subfigure}\hfill
	\begin{subfigure}[t]{.3\textwidth}\vskip 0pt
	\includegraphics[width=\linewidth]{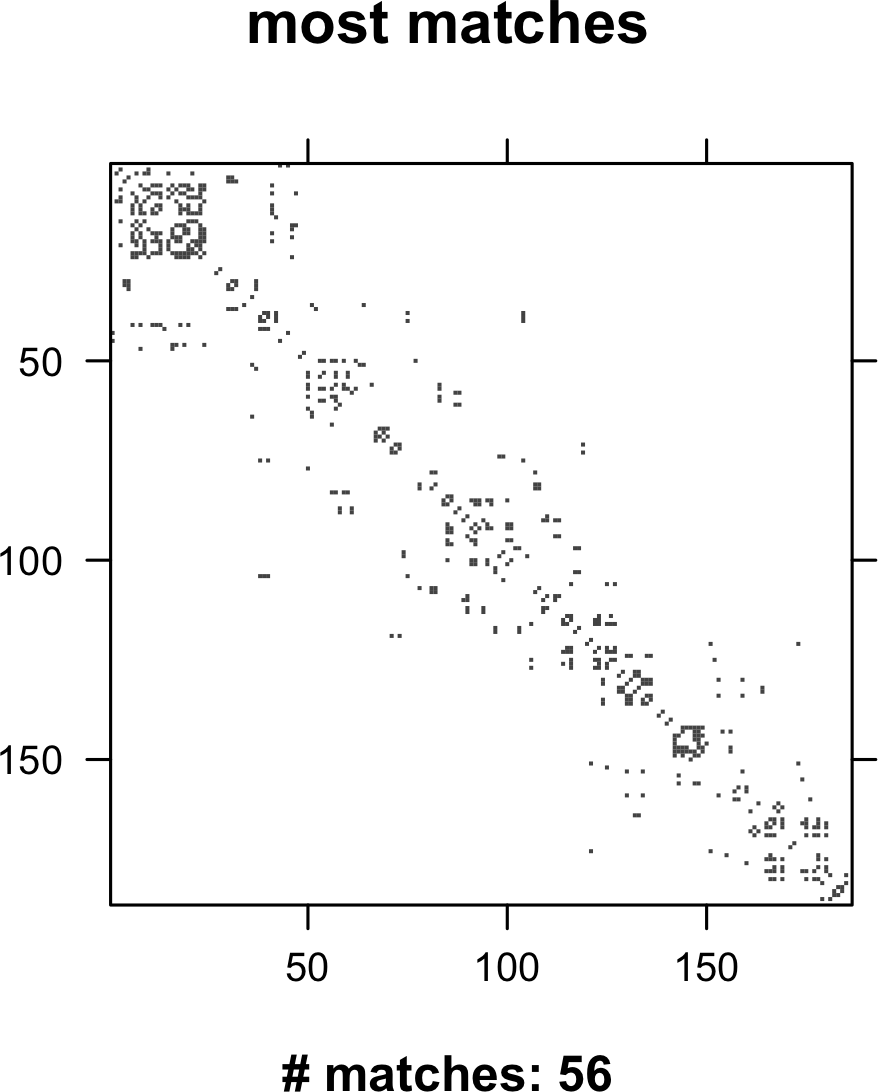}\\[3pt]
	\includegraphics[width=\linewidth]{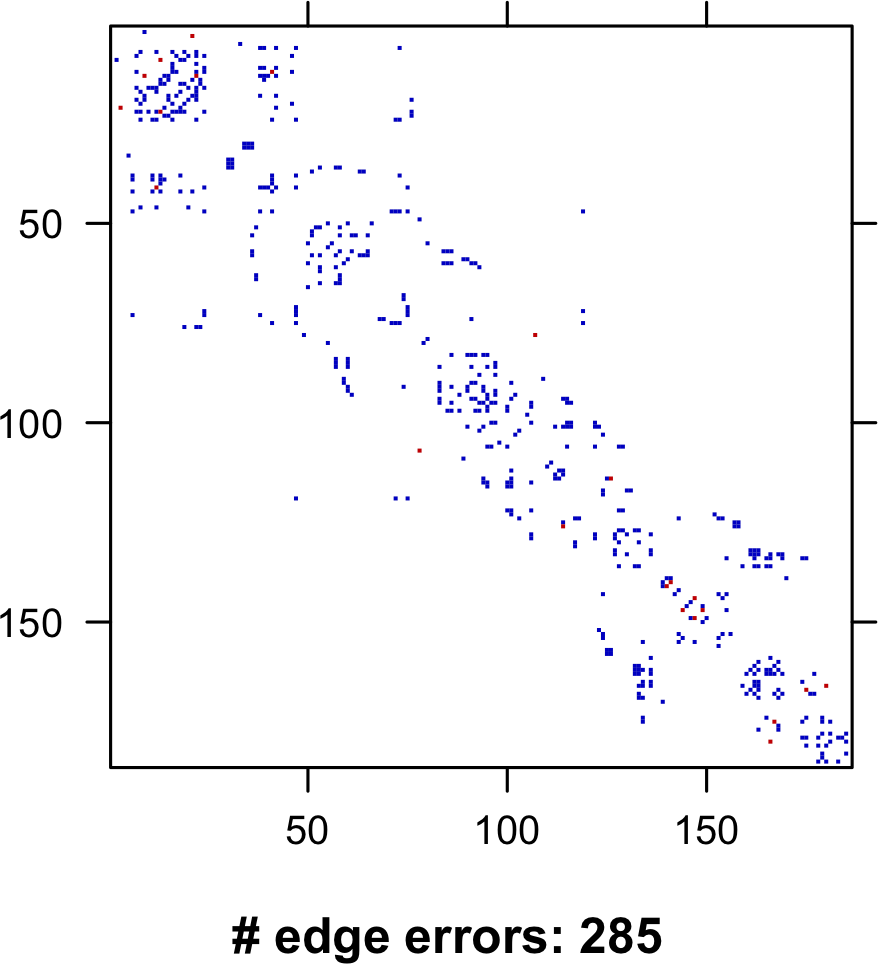}
	\caption{match with the most correct matches}
	\end{subfigure}
	\caption{Example matches from human DTMRI graphs.
	The top panels each are showing the adjacency matrices for graphs, with each row and column corresponding to nodes and dark pixels indicating an edge between them
	The bottom two panels show the difference between original graph in the center and the other two graphs.
	Blue pixels indicate edges missing in the matched graph and red pixels indicates extra edges in the matched graph.
	(a) The adjacency and difference matrices for the match with the lowest objective function among 1200 random restarts.
	(b) The graph corresponding to region 7 from the left hemisphere which we matched to the graph for the entire right hemisphere.
	(c) The adjacency and difference matrices for the match with the most correct matches among 1200 random restarts.}
	\label{fig:dt_example}
\end{figure*}

\section{Discussion}\label{sec:dis}

In this manuscript we have proposed a number of padding methods to transform the noisy subgraph detection problem into a graph matching problem. 
Each padding scheme emphasizes a different aspects of common graph structure.
Theoretically, the centered and oracle padding schemes are guaranteed to recover the original vertex correspondence in the correlated \ErdosRenyi{} statistical model, provided correlations and edge probabilities do not decay too quickly and the order of the smaller graph is at least logarithmic in the order of the larger graph.

By using a partially known vertex correspondence and random restarts, we are frequently able to recover the full correspondence accurately in a number of synthetic data settings and in a connectomics setting, using a gradient ascent approach on a relaxation of the original problem.
Importantly, we frequently observe a gap in the objective function value between a restart that achieves good performance and those with poor performance \cite{finke1987quadratic,Burkard1999-mz}.
If the size of the gap can be predicted, it can used to detect the fact that the correct alignment has been found rather than a spurious local minimum, and additionally can provide an adaptive stopping criterion for performing random restarts.

Rigorous results for the expected gap size in the correlated \ErdosRenyi{ model are currently unknown.
Two recent results for graphs with the same number of vertices do provide some hints.
\cite{Fishkind2018-sv} investigates the ratio of $\|A-B\|_F^2$ and $\Ex_P[\|A-PBP\|^T]$, where the expectation is take with respect to a uniform distribution over all permutation matrices.
The authors show that this quantity will converge to $1-\rho_T$, where $\rho_T$ is the ``total correlation'' of the graph pair, incorporating the correlation of the edge probabilities and the correlation of the edge presence random variables.
This result implies that $\|A-B\|_F^2 \approx (1-\rho_T) \Ex_P[\|A-PBP^T\|F^2]$.
Of course, the error for a local optimum will not necessarily be close to the error for a random permutation, so this result is not directly applicable.
Another related result is in~\cite{Fang2018-er}, which shows that under certain conditions there will be no local optimum which correctly align more than $\theta(\sqrt{n})$ vertices.
This gap in the number of aligned vertices presumably leads to a gap in the objective function.
Another related avenue is that for quadratic assignment programs for independent homogeneous \ErdosRenyi graph, the deviation of the objective function at the solution compared to the objective at any fixed permutation has been well studied \cite{Burkard1999-mz,finke1987quadratic}.
Further explorations along these lines are active areas of research.}

As the best objective function does not always correspond to the best matching performance, it can be difficult to evaluate performance when ground truth is not present. 
One way to quantify uncertainty is to study the proportion of times that two nodes are matched across graphs among the random restarts. 
By essentially taking the average of the permutations, one could determine which matches are most probable and which have higher uncertainty.
We demonstrate this idea in Figure~\ref{fig:kc_matched_pair} where we see patterns of matches indicating that certain groups of vertices are more similar within the set of K-cells.

We have assumed that the induced subgraphs are identically distributed and positively correlated.
These assumptions are important for our proof.
Indeed, if negative correlation is allowed, the objective functions are seeking the wrong solution.
Non-identical graphs can be more easily dealt with by using separate oracle estimates, as was done in the random dot product and Drosophila connectome examples.
Further theoretical issues such as non-independent edge graph may also be possible to deal with by exploiting the possibility that the graph pair is constructed out of a still large number of independent Bernoulli trials, none of which can drastically change graph structure.
Such an assumption may allow for the application of Proposition~\ref{prop:kim} as in our proofs below.

The general question of which pairs of graphs are matchable, by which we mean the ``true'' correspondence is recoverable, remains open.
One of the key aspects that make it possible in the correlated \ErdosRenyi{} model is that the noise introduced by the independent edges introduces asymmetries in the observed graphs, so that there are few or no automorphisms or even near automorphisms.
Hence, under the correlated setting, all but the true correspondence lead to an increase in the number of edge discrepancies.
For two arbitrary graphs, it is reasonable to expect that a similar conditions would be required, namely that there is a sufficient agreement between the graphs at the true correspondence, and that each graph is sufficiently asymmetric so that no other correspondence also has a small number of edge disagreements.

\clearpage


\section*{Acknowledgments}
This material is based on research sponsored by the Air Force Research Laboratory and DARPA under agreement number FA8750-18-2-0066. The U.S. Government is authorized to reproduce and distribute reprints for Governmental purposes notwithstanding any copyright notation thereon.  The views and conclusions contained herein are those of the authors and should not be interpreted as necessarily representing the official policies or endorsements, either expressed or implied, of the Air Force Research Laboratory and DARPA or the U.S. Government.
Vince Lyzinski also gratefully acknowledge the support of NIH grant BRAIN U01-NS108637
This work is also partially supported by a grant from MIT
Lincoln Labs.
\bibliographystyle{IEEEtran}
\bibliography{../biblio.bib}

\begin{thebibliography}{1}
\providecommand{\url}[1]{#1}
\csname url@samestyle\endcsname
\providecommand{\newblock}{\relax}
\providecommand{\bibinfo}[2]{#2}
\providecommand{\BIBentrySTDinterwordspacing}{\spaceskip=0pt\relax}
\providecommand{\BIBentryALTinterwordstretchfactor}{4}
\providecommand{\BIBentryALTinterwordspacing}{\spaceskip=\fontdimen2\font plus
\BIBentryALTinterwordstretchfactor\fontdimen3\font minus
  \fontdimen4\font\relax}
\providecommand{\BIBforeignlanguage}[2]{{%
\expandafter\ifx\csname l@#1\endcsname\relax
\typeout{** WARNING: IEEEtran.bst: No hyphenation pattern has been}%
\typeout{** loaded for the language `#1'. Using the pattern for}%
\typeout{** the default language instead.}%
\else
\language=\csname l@#1\endcsname
\fi
#2}}
\providecommand{\BIBdecl}{\relax}
\BIBdecl

\bibitem{kim}
J.~H. Kim, B.~Sudakov, and V.~H. Vu, ``On the asymmetry of random regular
  graphs and random graphs,'' \emph{Random Structures and Algorithms}, vol.~21,
  pp. 216--224, 2002.

\bibitem{arratia1989tutorial}
R.~Arratia and L.~Gordon, ``\BIBforeignlanguage{en}{Tutorial on large
  deviations for the binomial distribution},''
  \emph{\BIBforeignlanguage{en}{Bulletin of mathematical biology}}, vol.~51,
  no.~1, pp. 125--131, 1989.

\end{thebibliography}


\begin{thebibliography}{10}
\providecommand{\url}[1]{#1}
\csname url@samestyle\endcsname
\providecommand{\newblock}{\relax}
\providecommand{\bibinfo}[2]{#2}
\providecommand{\BIBentrySTDinterwordspacing}{\spaceskip=0pt\relax}
\providecommand{\BIBentryALTinterwordstretchfactor}{4}
\providecommand{\BIBentryALTinterwordspacing}{\spaceskip=\fontdimen2\font plus
\BIBentryALTinterwordstretchfactor\fontdimen3\font minus
  \fontdimen4\font\relax}
\providecommand{\BIBforeignlanguage}[2]{{%
\expandafter\ifx\csname l@#1\endcsname\relax
\typeout{** WARNING: IEEEtran.bst: No hyphenation pattern has been}%
\typeout{** loaded for the language `#1'. Using the pattern for}%
\typeout{** the default language instead.}%
\else
\language=\csname l@#1\endcsname
\fi
#2}}
\providecommand{\BIBdecl}{\relax}
\BIBdecl

\bibitem{Carletti2018-ea}
V.~Carletti, P.~Foggia, A.~Saggese, and M.~Vento,
  ``\BIBforeignlanguage{en}{Challenging the time complexity of exact subgraph
  isomorphism for huge and dense graphs with {VF3}},''
  \emph{\BIBforeignlanguage{en}{IEEE transactions on pattern analysis and
  machine intelligence}}, vol.~40, no.~4, pp. 804--818, Apr. 2018.

\bibitem{Kuramochi2001-ll}
M.~Kuramochi and G.~Karypis, ``Frequent subgraph discovery,'' in
  \emph{Proceedings 2001 {IEEE} International Conference on Data Mining}, 2001,
  pp. 313--320.

\bibitem{Slota2013-zg}
G.~M. Slota and K.~Madduri, ``Fast approximate subgraph counting and
  enumeration,'' in \emph{2013 42nd International Conference on Parallel
  Processing}, Oct. 2013, pp. 210--219.

\bibitem{Cordella2004-nc}
L.~P. Cordella, P.~Foggia, C.~Sansone, and M.~Vento,
  ``\BIBforeignlanguage{en}{A (sub)graph isomorphism algorithm for matching
  large graphs},'' \emph{\BIBforeignlanguage{en}{IEEE transactions on pattern
  analysis and machine intelligence}}, vol.~26, no.~10, pp. 1367--1372, Oct.
  2004.

\bibitem{Schmidt2009-nf}
M.~C. Schmidt, N.~F. Samatova, K.~Thomas, and B.-H. Park, ``A scalable,
  parallel algorithm for maximal clique enumeration,'' \emph{Journal of
  parallel and distributed computing}, vol.~69, no.~4, pp. 417--428, Apr. 2009.

\bibitem{Bomze1999-bx}
I.~M. Bomze, M.~Budinich, P.~M. Pardalos, and M.~Pelillo, ``The maximum clique
  problem,'' in \emph{Handbook of Combinatorial Optimization: Supplement Volume
  A}, D.-Z. Du and P.~M. Pardalos, Eds.\hskip 1em plus 0.5em minus 0.4em\relax
  Boston, MA: Springer US, 1999, pp. 1--74.

\bibitem{Ullmann1976-ir}
J.~R. Ullmann, ``An algorithm for subgraph isomorphism,'' \emph{Journal of the
  ACM}, vol.~23, no.~1, pp. 31--42, Jan. 1976.

\bibitem{ConteReview}
D.~Conte, P.~Foggia, C.~Sansone, and M.~Vento, ``Thirty years of graph matching
  in pattern recognition,'' \emph{International Journal of Pattern Recognition
  and Artificial Intelligence}, vol.~18, no.~03, pp. 265--298, 2004.

\bibitem{foggia2014graph}
P.~Foggia, G.~Percannella, and M.~Vento, ``Graph matching and learning in
  pattern recognition in the last 10 years,'' \emph{International Journal of
  Pattern Recognition and Artificial Intelligence}, vol.~28, no.~01, p.
  1450001, 2014.

\bibitem{Emmert-Streib2016-st}
F.~Emmert-Streib, M.~Dehmer, and Y.~Shi, ``Fifty years of graph matching,
  network alignment and network comparison,'' \emph{Information sciences}, vol.
  346--347, pp. 180--197, 2016.

\bibitem{lyzinski_spectral}
V.~Lyzinski, D.~L. Sussman, D.~E. Fishkind, H.~Pao, L.~Chen, J.~T. Vogelstein,
  Y.~Park, and C.~E. Priebe, ``Spectral clustering for divide-and-conquer graph
  matching,'' \emph{Parallel Computing}, vol.~47, pp. 70--87, 2015.

\bibitem{yartseva2013performance}
L.~Yartseva and M.~Grossglauser, ``On the performance of percolation graph
  matching,'' in \emph{Proceedings of the first ACM conference on Online social
  networks}.\hskip 1em plus 0.5em minus 0.4em\relax ACM, 2013, pp. 119--130.

\bibitem{Akoglu2015-xe}
L.~Akoglu, H.~Tong, and D.~Koutra, ``\BIBforeignlanguage{en}{Graph based
  anomaly detection and description: a survey},''
  \emph{\BIBforeignlanguage{en}{Data mining and knowledge discovery}}, vol.~29,
  no.~3, pp. 626--688, May 2015.

\bibitem{FAP}
D.~Fishkind, S.~Adali, H.~G. Patsolic, L.~Meng, V.~Lyzinski, and C.~Priebe,
  ``Seeded graph matching,'' \emph{arXiv preprint arXiv:1209.0367}, 2017.

\bibitem{FAQ}
J.~T. {Vogelstein}, J.~M. {Conroy}, V.~{Lyzinski}, L.~J. {Podrazik}, S.~G.
  {Kratzer}, E.~T. {Harley}, D.~E. {Fishkind}, R.~J. {Vogelstein}, and C.~E.
  {Priebe}, ``{Fast Approximate Quadratic Programming for Graph Matching},''
  \emph{PLoS ONE}, vol.~10, no.~04, 2014.

\bibitem{Lyzinski2016-kp}
V.~Lyzinski, D.~E. Fishkind, M.~Fiori, J.~T. Vogelstein, C.~E. Priebe, and
  G.~Sapiro, ``Graph matching: Relax at your own risk,'' \emph{IEEE
  transactions on pattern analysis and machine intelligence}, vol.~38, no.~1,
  pp. 60--73, 2016.

\bibitem{Samsi2017-if}
S.~Samsi, V.~Gadepally, M.~Hurley, M.~Jones, E.~Kao, S.~Mohindra,
  P.~Monticciolo, A.~Reuther, S.~Smith, W.~Song, D.~Staheli, and J.~Kepner,
  ``Static graph challenge: Subgraph isomorphism,'' in \emph{2017 {IEEE} High
  Performance Extreme Computing Conference ({HPEC})}, 2017, pp. 1--6.

\bibitem{Caelli2004-nd}
T.~Caelli and S.~Kosinov, ``\BIBforeignlanguage{en}{An eigenspace projection
  clustering method for inexact graph matching},''
  \emph{\BIBforeignlanguage{en}{IEEE transactions on pattern analysis and
  machine intelligence}}, vol.~26, no.~4, pp. 515--519, Apr. 2004.

\bibitem{Dutta2017-qg}
A.~Dutta, J.~Llad{\'o}s, H.~Bunke, and U.~Pal, ``Product graph-based higher
  order contextual similarities for inexact subgraph matching,'' Feb. 2017.

\bibitem{Carletti2016-yn}
V.~Carletti, ``Exact and inexact methods for graph similarity in structural
  pattern recognition {PhD} thesis of vincenzo carletti,'' Ph.D. dissertation,
  Universit{\'e} de Caen; Universita degli studi di Salerno, 2016.

\bibitem{Ghahraman1980-pg}
D.~E. Ghahraman, A.~K.~C. Wong, and T.~Au, ``Graph optimal monomorphism
  algorithms,'' \emph{IEEE transactions on systems, man, and cybernetics},
  vol.~10, no.~4, pp. 181--188, Apr. 1980.

\bibitem{zaslavskiy2009path}
M.~Zaslavskiy, F.~Bach, and J.-P. Vert, ``A path following algorithm for the
  graph matching problem,'' \emph{IEEE Transactions on Pattern Analysis and
  Machine Intelligence}, vol.~31, no.~12, pp. 2227--2242, 2009.

\bibitem{jovo}
M.~Fiori, P.~Sprechmann, J.~Vogelstein, P.~Musé, and G.~Sapiro, ``Robust
  multimodal graph matching: Sparse coding meets graph matching,''
  \emph{Advances in Neural Information Processing Systems}, pp. 127--135, 2013.

\bibitem{Egozi2013-jh}
A.~Egozi, Y.~Keller, and H.~Guterman, ``\BIBforeignlanguage{en}{A probabilistic
  approach to spectral graph matching},'' \emph{\BIBforeignlanguage{en}{IEEE
  transactions on pattern analysis and machine intelligence}}, vol.~35, no.~1,
  pp. 18--27, Jan. 2013.

\bibitem{alex}
Y.~Aflalo, A.~Bronstein, and R.~Kimmel, ``On convex relaxation of graph
  isomorphism,'' \emph{Proceedings of the National Academy of Sciences}, vol.
  112, no.~10, pp. 2942--2947, 2015.

\bibitem{JMLR:v15:lyzinski14a}
V.~Lyzinski, D.~E. Fishkind, and C.~E. Priebe, ``Seeded graph matching for
  correlated {E}rdos-{R}enyi graphs,'' \emph{Journal of Machine Learning
  Research}, vol.~15, pp. 3513--3540, 2014.

\bibitem{Fang2018-er}
F.~Fang, D.~L. Sussman, and V.~Lyzinski, ``Tractable graph matching via soft
  seeding,'' \emph{arXiv preprint arXiv:1807.09299}, 2018.

\bibitem{Hoff2009-pj}
P.~D. Hoff, ``Multiplicative latent factor models for description and
  prediction of social networks,'' \emph{Computational \& Mathematical
  Organization Theory}, vol.~15, no.~4, pp. 261--272, 2009.

\bibitem{Hoff2002-lj}
P.~D. Hoff, A.~E. Raftery, and M.~S. Handcock, ``Latent space approaches to
  social network analysis,'' \emph{Journal of the American Statistical
  Association}, vol.~97, no. 460, pp. 1090--1098, 2002.

\bibitem{Durante2016-tk}
D.~Durante, D.~B. Dunson, and J.~T. Vogelstein, ``Nonparametric bayes modeling
  of populations of networks,'' \emph{Journal of the American Statistical
  Association}, pp. 1--15, Aug. 2016.

\bibitem{chatterjee2014matrix}
S.~Chatterjee, ``Matrix estimation by universal singular value thresholding,''
  \emph{The Annals of Statistics}, vol.~43, no.~1, pp. 177--214, 2014.

\bibitem{Davenport2014-tf}
M.~A. Davenport, Y.~Plan, E.~van~den Berg, and M.~Wootters, ``1-bit matrix
  completion,'' \emph{Information and Inference: A Journal of the IMA}, vol.~3,
  no.~3, pp. 189--223, Sep. 2014.

\bibitem{Bunke1997-bk}
H.~Bunke, ``On a relation between graph edit distance and maximum common
  subgraph,'' \emph{Pattern recognition letters}, vol.~18, no.~8, pp. 689--694,
  Aug. 1997.

\bibitem{lyzinski2017graph}
V.~Lyzinski and D.~L. Sussman, ``Matchability of heterogeneous networks
  pairs,'' \emph{arXiv preprint, arXiv:1705.02294}, 2017.

\bibitem{Lovasz2012-df}
L.~Lov{\'a}sz, \emph{Large networks and graph limits}.\hskip 1em plus 0.5em
  minus 0.4em\relax American Mathematical Society Providence, 2012, vol.~60.

\bibitem{alon1995color}
N.~Alon, R.~Yuster, and U.~Zwick, ``Color-coding,'' \emph{Journal of the ACM
  (JACM)}, vol.~42, no.~4, pp. 844--856, 1995.

\bibitem{Munkres1957-ny}
J.~Munkres, ``Algorithms for the assignment and transportation problems,''
  \emph{Journal of the Society for Industrial and Applied Mathematics}, vol.~5,
  no.~1, pp. 32--38, 1957.

\bibitem{young2007random}
S.~Young and E.~Scheinerman, ``Random dot product graph models for social
  networks,'' in \emph{Proceedings of the 5th international conference on
  algorithms and models for the web-graph}, 2007, pp. 138--149.

\bibitem{nickel2006random}
C.~L.~M. Nickel, ``Random dot product graphs: A model for social networks,''
  Ph.D. dissertation, Johns Hopkins University, 2006.

\bibitem{eichler2017complete}
K.~Eichler, F.~Li, A.~Litwin-Kumar, Y.~Park, I.~Andrade, C.~M.
  Schneider-Mizell, T.~Saumweber, A.~Huser, C.~Eschbach, B.~Gerber
  \emph{et~al.}, ``The complete connectome of a learning and memory centre in
  an insect brain,'' \emph{Nature}, vol. 548, no. 7666, 2017.

\bibitem{Kiar2018-ti}
G.~Kiar, E.~Bridgeford, W.~G. Roncal, V.~Chandrashekhar, D.~Mhembere, R.~Burns,
  and J.~Vogelstein, ``neurodata/ndmg: Stable {ndmg-DWI} pipeline release,''
  Jan. 2018.

\bibitem{Kiar2018-fl}
G.~Kiar, E.~Bridgeford, W.~G. Roncal, {Consortium for Reliability and
  Reproducibliity (CoRR)}, V.~Chandrashekhar, D.~Mhembere, S.~Ryman, X.-N. Zuo,
  D.~S. Marguiles, R.~Cameron~Craddock, C.~E. Priebe, R.~Jung, V.~Calhoun,
  B.~Caffo, R.~Burns, M.~P. Milham, and J.~Vogelstein,
  ``\BIBforeignlanguage{en}{A {High-Throughput} pipeline identifies robust
  connectomes but troublesome variability},'' Apr. 2018.

\bibitem{finke1987quadratic}
G.~Finke, R.~E. Burkard, and F.~Rendl, ``Quadratic assignment problems,''
  \emph{North-Holland Mathematics Studies}, vol. 132, pp. 61--82, 1987.

\bibitem{Burkard1999-mz}
R.~E. Burkard, E.~{\c C}ela, P.~M. Pardalos, and L.~S. Pitsoulis, ``The
  quadratic assignment problem,'' in \emph{Handbook of Combinatorial
  Optimization: Volume1--3}, D.-Z. Du and P.~M. Pardalos, Eds.\hskip 1em plus
  0.5em minus 0.4em\relax Boston, MA: Springer US, 1999, pp. 1713--1809.

\bibitem{Fishkind2018-sv}
D.~E. Fishkind, L.~Meng, A.~Sun, C.~E. Priebe, and V.~Lyzinski, ``Alignment
  strength and correlation for graphs,'' Aug. 2018.

\end{thebibliography}

\clearpage
\appendix

\section{Auxiliary Results}

We use the following McDiarmid-like concentration result in the proof of each of our theorems.
\begin{proposition}[Proposition 3.2 from~\citeapp{kim}]\label{prop:kim}
Let $X_1,\dotsc,X_m$ be a sequence of independent Bernoulli random variables where $\Ex[X_i]=p_i$. 
Let $f:\{0,1\}^m\mapsto \Re$ be such that changing any $X_i$ to $1-X_i$ changes $f$ by at most 
\[ 
     \begin{split}M=\sup_i \sup_{X_1,\dotsc,X_n}
    |&f(X_1,\dotsc,X_i,\dotsc,X_n) \\&- f(X_1,\dotsc,1-X_i,\dotsc,X_n)|
    \end{split}.
\]
Let $\sigma^2 = M^2 \sum_i p_i(1-p_i)$ and let $Y=f(X_1,\dotsc,X_n)$.

Then $$\Pr[|Y-\Ex[Y]| \geq t \sigma ] \leq 2 e^{-t^2/4}$$ for all $0<t<2\sigma/M$.
\end{proposition}

The following proposition characterizes how a pair of correlated Bernoulli random variables can be generated from a set of three independent random variables.
Along with a bound on the variances of these three Bernoullis this allows for the use of Proposition~\ref{prop:kim}.

\begin{proposition}
\label{prop2}
If $X,Y$ are marginally Bernoulli random variables with parameter $\Lambda$ and correlation $\rho$, then the bivariate Bernoulli distribution of $(X,Y)$ is given by
\begin{center}
  \begin{tabular}{c|c|c|c}
    
    $(X,Y) $ & $X=1$ & $X=0$& \textup{Total}\\
    \hline
    $Y=1$ & $\Lambda[\Lambda+\rho(1-\Lambda)]$ &  $(1-\Lambda)\Lambda(1-\rho)$ &  $\Lambda$\\
    \hline
    $Y=0$ & $\Lambda(1-\rho)(1-\Lambda)$ & $(1-\Lambda+\Lambda \rho)(1-\Lambda)$& $1-\Lambda$\\
    \hline
    \textup{Total} & $\Lambda$ & $1-\Lambda$ & 1\\
  \end{tabular}
\end{center}

Let $Z_0, Z_1$ and $Z_2$ be independent Bernoulli random variables with $Z_0 \sim \mathrm{Bern}(\Lambda)$, $Z_1\sim \mathrm{Bern} \left(\Lambda(1-\rho)\right)$ and $Z_2\sim \mathrm{Bern}\left(\mathrm{\Lambda+\rho(1-\Lambda)}\right)$. Then $$(X,Y) \overset{d}{\sim} (Z_0, (1-Z_0)Z_1+Z_0Z_2).$$ 
It also holds that
\begin{align*}
\mathrm{Var}(Z_{0})+\mathrm{Var}(Z_{1})+ \mathrm{Var}(Z_{2}) 
& \leq 3\Lambda(1-\Lambda)+\rho(1-\rho) \\
\mathrm{Var}(Z_{0})+\mathrm{Var}(Z_{1})+ \mathrm{Var}(Z_{2}) 
& \geq 3(1-\rho)\Lambda(1-\Lambda).
\end{align*} 
\end{proposition}

\section{Proofs of Main Results}\label{app:proof}
In this section we prove our three main theorems.

For convenience, we recall the main variables for each theorem.
Let $[n]=\{1,2,\dotsc,n\}$.
Let $\PM_n$ and $\DSM_n$ denote the set of $n\times n$ permutation matrices and doubly stochastic matrices, respectively.
Let $\1_n$ and $\0_n$ denote the $n\times n$ all-ones and all-zeros matrices, respectively, and $\1_{nm}$ and $\0_{nm}$ for $n\times m$ rectangular versions.
Let $\AM_n$ denote the set of adjacency matrices corresponding to simple undirected graphs, i.e.\ the set of symmetric hollow $\{0,1\}$-valued $n\times n$ matrices.
Finally, let $\oplus$ denote the direct sum of matrices.

The number of vertices is $n$ in the large graph and $n_c$ in the small graph.
The matrix of correlations is $R \in [0,1]^{n_c \times n_c}$ and the matrix of edge probabilities is $\Lambda\in [0,1]^{n_c\times n_c}$ with the principal $n_c \times n_c$ submatrix given by $\Lambda^c$.
We observe $A,B \sim \CER(\Lambda, R)$ with $A\in  \AM_{n_c}$ and $B \in \AM_{n}$.


\begin{proof}[Proof of Theorem~\ref{thm:naive}]
Let $Q\in\PM_n$ (with associated permutation $\tau$) map all vertices in the small graph to vertices without matches in the large graph, so that $k_j=|\{i\leq n_c: i\neq\tau(i)>n_c\}|=n_c$.
Let $\mathcal{E}_Q$ be the set of edges in $A$ permuted by $Q$, and note that $\mathcal{E}_Q=\binom{[n_c]}{2}.$

Let $X_Q=\tr(\tA\tB-\tA Q \tB Q^T)$.
We will construct an appropriate $\Lambda$ such that $\Pr[X_Q\geq 0] \leq \exp\{ - C \epsilon^2 (\log n_c)^2\}$, which implies the desired result. 

Note that for $u,v\in [n_c]$, $u\neq v$, $w,r\in[n]$, and $w\neq r$, we have 
$$\Ex(\tA_{u,v}\tB_{w,r})=\begin{cases}
\Lambda_{u,v}(\Lambda_{u,v}+(1-\Lambda_{u,v})R_{u,v})&\text{ if }\{u,v\}=\{w,r\},\\
\Lambda_{u,v}\Lambda_{w,r}&\text{ if }\{u,v\}\neq\{w,r\},
\end{cases}
$$
so that
\begin{align*}
\Ex X_Q&=4\left(\sum_{\{u,v\}\in\mathcal{E}_Q}\Lambda_{u,v}(\Lambda_{u,v}+(1-\Lambda_{u,v})R_{u,v})-\Lambda_{u,v}\Lambda_{\tau(u),\tau(v)} \right)\notag\\
&\leq4\left(\sum_{\{u,v\}\in\mathcal{E}_Q}\Lambda_{u,v}(\Lambda_{u,v}+(1-\Lambda_{u,v})\varrho)-\Lambda_{u,v}\Lambda_{\tau(u),\tau(v)} \right)\notag
\end{align*}
Letting each $\Lambda_{\tau(u),\tau(v)}>\beta+(1-\beta)\varrho+\epsilon$ be chosen to keep $\Lambda_{\tau(u),\tau(v)}\in(0,1)$ (which is possible by the assumption that $\beta$ and $\varrho$ are both strictly less than 1), we have
\begin{align*}
-\Ex X_Q&\geq4\left(\sum_{\{u,v\}\in\mathcal{E}_Q}\Lambda_{u,v}(\beta+(1-\beta)\varrho+\epsilon)-\Lambda_{u,v}(\Lambda_{u,v}+(1-\Lambda_{u,v})\varrho) \right)\\
&=\omega(\epsilon\, n_c\log n_c).
\end{align*}
Applying Proposition \ref{prop:kim} with $M=8$, $\sigma^2=\Theta(n_c^2)$, we have
\begin{align*}
\Pr\left[X_Q\geq 0  \right]&=\Pr\left[X_Q-\Ex X_Q\geq -\Ex X_Q  \right]\\
&\leq \Pr\left[|X_Q-\Ex X_Q|\geq C \epsilon\,n_c\log n_c  \right]\\
&\leq \exp\left\{-C(\epsilon \log n_c)^2 \right\},
\end{align*}
where $C$ is a positive constant that can vary line-to-line.
\end{proof}


\begin{proof}[Proof of Theore~\ref{thm:center}]
Let $\GM=\{Q\in \PM: Q^c=I_{n_c}\}$, where $Q^c$ denotes the $n_c\times n_c$ principal submatrix of $Q$, denote the set of permutation which correctly match the $n_c$ core vertices.
Note that if $P\in \GM$, then
$\|\tA -\tB\|_F=\|\tA-P\tB P^T\|_F.$

Let $Q\notin\GM$ (with associated permutation $\tau$) satisfy
$k_j = |\{ i \leq n_c : i \neq \tau(i) > n_c \}|$ 
and 
$k_c = |\{i\leq n_c: i\neq \tau(i)\leq n_c|$,
so that $Q$ correctly matches the labels of $n_c - k_c - k_j$ of the core $n_c$ vertices across $A$ and $B$.
Define $X_Q=\frac{1}{2}\tr(\tA\tB-\tA Q\tB Q^T)$ and 
$$\mathcal{E}_Q=\left\{\{u,v\}\in\binom{[n_c]}{2}:\{\tau(u),\tau(v)\}\neq \{u,v\} \right\}.$$
Note that $|\mathcal{E}_Q|\geq\frac{(n_c-1)(k_c+k_j)}{2}$.
We then have
\begin{align*}
\Ex X_Q&=\frac{1}{2}\Ex\left(\tr(\tA^T \tB - \tA^T Q \tB Q^T)\right)\\
&=\sum_{\{u,v\}\in\mathcal{E}_Q}\Ex(A_{u,v}B_{u,v} - A_{u,v}B_{\tau(u),\tau(v)}) .
\end{align*}
For $u, v \in [n_c]$, $u \neq v$, $w,r\in[n]$, and $w\neq r$, we have 
$$\Ex(\tA_{u,v}\tB_{w,r})=\begin{cases}
(1-2\Lambda_{u,v})^2+4\Lambda_{u,v}(1-\Lambda_{u,v})R_{u,v}&\text{ if }\{u,v\}=\{w,r\},\\
(1-2\Lambda_{u,v})(1-2\Lambda_{w,r})&\text{ if }\{u,v\}\neq\{w,r\}.
\end{cases}
$$
Hence,
\begin{align*}
\Ex X_Q&=\sum_{\{u,v\}\in\mathcal{E}_Q}\bigg((1-2\Lambda_{u,v})2(\Lambda_{\tau(u),\tau(v)}-\Lambda_{u,v})+4\Lambda_{u,v}(1-\Lambda_{u,v})R_{u,v} \bigg).
\end{align*}
Letting $\gamma=\alpha(1-\alpha)$, under the assumptions of the theorem we have that
\begin{align}
\label{eq:pad_entrywise_bound}
    2&(1-2\Lambda_{u,v})2(\Lambda_{\tau(u),\tau(v)}-\Lambda_{u,v})+4\Lambda_{u,v}(1-\Lambda_{u,v})R_{u,v}\notag\\
    &\geq 2\Lambda_{\tau(u),\tau(v)}-2\Lambda_{u,v}-4\Lambda_{u,v}\Lambda_{\tau(u),\tau(v)}+4\Lambda_{u,v}^2+2\Lambda_{u,v}-2\Lambda_{u,v}^2+4\epsilon\gamma\notag\\
    &\geq 2\Lambda_{\tau(u),\tau(v)}^2+2\Lambda_{u,v}^2-4\Lambda_{u,v}\Lambda_{\tau(u),\tau(v)}+4\epsilon\gamma\geq4\epsilon\gamma,
\end{align}
so that
\begin{align*}
\Ex X_Q&\geq 2\epsilon (n_c-1)(k_c+k_j)\gamma
\end{align*}
Note that $X_Q$ is a function of
$$N_Q:=3\left( (n_c-k_c-k_j)k_c+\binom{k_c}{2}\right)+2\left( (n_c-k_j)k_j+\binom{k_j}{2} \right)\leq 3n_c(k_c+k_j)$$
independent Bernoulli random variables, and in the language of Proposition \ref{prop:kim}, we have that $M=8$ and $\sigma^2$ satisfies $\sigma^2\leq 3 n_c(k_c+k_j)$ and 
$$
\sigma^2 \geq \left(\binom{n_c}{2}-\binom{ n_c-k_c-k_j }{2}\right)\gamma$$
Setting 
$$
t= \frac{8\epsilon (n_c-1)(k_c+k_j)\cdot\gamma}{\sigma}$$
yields $t<\sigma$ as required, and letting $C$ be a positive constant that can vary line by line we have that,
\begin{align}
\label{eq:padbound}
\Pr\left(X_Q\leq 0\right)&\leq\Pr\left(|X_Q-\Ex(X_Q)|\geq \Ex(X_Q)\right)\notag\\
&\leq\Pr\left(|X_P-\Ex(X_P)|\geq \frac{C\epsilon (n_c-1)(k_c+k_j)\gamma}{\sigma}\cdot\sigma\right)\notag\\
&\leq 2\exp\left\{-C\epsilon^2 n_c(k_c+k_j)\gamma^2\right\}.
\end{align}
Define the equivalence relation ``$\sim$'' on $\PM{}_{n}$ via
$P\sim Q$ if $P_{\cdot,v}=Q_{\cdot,v}$ for all core vertices $v$.  Note that $P\sim Q$ implies that $X_P=X_Q$.
To prove that there is no $P\notin\GM$ satisfying
$X_P\leq0$, it suffices to consider a single $P$ from each equivalence class under ``$\sim$''.
Summing Eq.~\eqref{eq:padbound} over the (at most) $n_c^{2(k_j+k_c)}n_j^{2k_j}$ equivalence classes for each $k_c,k_j$ and then over $k_c$ and $k_j$ yields the desired result.
\end{proof}


\begin{proof}[\added{Proof of Theorem~\ref{thm:lowbnd2}}]
\added{For each $i\in[n-n_c]$, define $P_i$ to be the permutation matrix whose associated permutation transposes $n_c+i$ and $n_c$.
Let $X_i=\tr(AP_i BP_i)^T/2$ and let $X=\tr(AB)^T/2$.
It holds that $X < X_i$ if and only if 
$$\sum_{v<n_c} \tilde A_{n_c,v} (\tilde B_{n_c,v}-\tilde B_{n_c+i, v}) < 0,$$
where $\tilde A_{i,j}$ and $\tilde B_{i,j}$ are $1$ (resp., $-1$) depending on the presence (resp., absence) of an edge between $i$ and $j$ in $A$ or $B$.
Let $Z=\sum_{i=1}^{n-m} \mathbf{1}_{\{X<X_i\}}$. 
We will proceed to show that, under the assumptions on the growth rates of $n$ and $n_c$, we have that $\Pr[Z=0]=o(1)$.}

\added{Let the event $\mathcal{E}_{\vec a,\vec b}$ be defined as
For $\vec{a},\vec{b} \in \{0,1\}^{n_c-1}$, let $$\mathcal{E}_{\vec a,\vec b}=\{\tilde A_{n_c,v} = a_v,\,\tilde B_{n_c,v} = b_v\text{ for }v\in [n_c - 1]\},$$ and let }
\begin{align*}
w'=|\{v:a_v=1\}|;  \,w&=|\{v:a_v=b_v=1\}|; \\
y'=|\{v:a_v=-1\}|; \,y&=|\{v:a_v=b_v=-1\}|.
\end{align*}
\added{Conditional on $\mathcal{E}_{\vec a,\vec b}$,
$X$ and $X_i$ then decompose into}
\begin{align*}
X&=w-(w'-w)+y-(y'-y)\\
X_i&=W_i-(w'-W_i)+Y_i-(y'-Y_i),
\end{align*}
\added{where $W_i=\sum_{v:a_v=1}W_{iv}$ and $Y_i=\sum_{v:a_v=-1}Y_{iv}$ are independent random variables defined via
$$W_{iv}\stackrel{ind.}{\sim}\text{Bern}(\Lambda_{n_c+i,v}),\text{ and }Y_{iv}\stackrel{ind.}{\sim}\text{Bern}(1-\Lambda_{n_c+i,v}).$$
This yields that $(X-X_i) | \mathcal{E}_{\vec a,\vec b}=2(w-W_i)+2(y-Y_i)$, and that 
$$\{(X-X_i) | \mathcal{E}_{\vec a,\vec b}\}_{i=1}^{n-n_c}$$ are independent random variables.}

\added{Note that}
\begin{align*}
\mathbf{E}&(\#\,v\in[n_c]\text{ s.t. }\, A_{n_c,v}=1)/ n_c \in \left[\alpha,1-\alpha\right]\\
\mathbf{E}&(\#\,v\in[n_c]\text{ s.t. }\, A_{n_c,v}=-1)/ n_c \in \left[\alpha,1-\alpha\right]\\
\mathbf{E}&(\#\,v\in[n_c]\text{ s.t. }\, A_{n_c,v}=B_{n_c,v}=1)\\
&\hspace{10mm}/ n_c \in \left[\alpha(\alpha+(1-\alpha)\alpha),(1-\alpha)(1-\alpha+\alpha(1-\alpha))\right]\\
\mathbf{E}&(\#\,v\in[n_c]\text{ s.t. }\, A_{n_c,v}=B_{n_c,v}=-1)\\
&\hspace{10mm}/ n_c \in \left[\alpha(\alpha+(1-\alpha)\alpha),(1-\alpha)(1-\alpha+\alpha(1-\alpha))\right]
\end{align*}
\added{Binomial concentration inequalities (for example, Hoeffding's inequality) then yield that for the event $\mathcal{E} = \cup_{\vec{a},\vec{b}}\mathcal{E}_{\vec{a},\vec{b}}$, where the union is over all $\vec{a},\vec{b}$ with }
\begin{align}
\label{eq:alphaalphaalpha}
w'&\in n_c \left[\frac{\alpha}{2},1-\frac{\alpha}{2}\right];\\
w&\in n_c \left[\frac{\alpha}{2}\alpha(2-\alpha),\left(1-\frac{\alpha}{2}\right)(1-\alpha^2)\right]; \\
y'&\in n_c \left[\frac{\alpha}{2},1-\frac{\alpha}{2}\right];\\
\label{eq:alphaalphaalpha4}
y&\in n_c \left[\frac{\alpha}{2}\alpha(2-\alpha),\left(1-\frac{\alpha}{2}\right)(1-\alpha^2)\right],
\end{align}
it holds that $\Pr[\mathcal{E}] = 1-o(1)$.

\added{
Conditioning on $\mathcal{E}_{\vec a,\vec b}$ for $\vec a$ and $\vec b$ satisfying Eq. (\ref{eq:alphaalphaalpha})--(\ref{eq:alphaalphaalpha4}), we have that
$$
W_i\stackrel{stoc.}{\geq}W'_i\sim Bin(n_c\alpha/2,\alpha),\,\,
Y_i\stackrel{stoc.}{\geq}Y'_i\sim Bin(n_c\alpha/2,\alpha).
$$
A standard large deviations bound for the Binomial distribution (see, for example, Theorem 2 of}~\citeapp{arratia1989tutorial})
\added{then yields}
\begin{align*}
\mathbb{P}(X_i>X|\mathcal{E}_{\vec a,\vec b})
&\geq \mathbb{P}(W_i>w,Y_i>y|\mathcal{E}_{\vec a,\vec b})\\
& \geq \mathbb{P}(W_i'>w,Y_i'>y|\mathcal{E}_{\vec a,\vec b})\\
&\geq C\frac{1}{\sqrt{n_c}}e^{-Dn_c}
\end{align*}
\added{for positive constants $C,D$ that are independent of $i$ (indeed, $C$ and $D$ only depend on $\alpha$).
The second moment method yields that }
\begin{align*}
\mathbb{P}(Z=0|\mathcal{E}_{\vec a,\vec b})&\leq \frac{\text{Var}(Z|\mathcal{E}_{\vec a,\vec b}) }{\mathbb{E}(Z|\mathcal{E}_{\vec a,\vec b})^2 }\\
&\leq \frac{n}{\left((n-n_c)C\frac{1}{\sqrt{n_c}}e^{-Dn_c}  \right)^2 }\\
& =\theta\left(\frac{n_ce^{2Dn_c}}{n }\right).
\end{align*}
\added{Therefore, there exists a constant $\xi>0$ such that $n_c <\xi \log n$ implies that
$\mathbb{P}(Z > 0|\mathcal{E}_{\vec a,\vec b})=1-o(1).$
Unconditioning $\mathcal{E}_{\vec a,\vec b}$ yields the desired result.}
\end{proof}


\begin{proof}[Proof of Theorem~\ref{thm:oracle}]
The proof for the oracle padding scheme proceeds almost identically to the proof for the centered padding scheme.
In particular, for each $Q$ we let $X_Q$, $\mathcal{E}_Q$, $k_j$ and $k_c$ be defined as in the proof of Theorem~\ref{thm:center}.
Hence, using that $\Ex[\tA_{uv}\tB_{wr}] = \begin{cases}
    R_{uv} \Lambda_{uv} (1 - \Lambda_{uv}) , &\text{ if } \{u, v \} = \{w, r \} \\
    0, &\text{ if } \{u, v \} \neq \{w, r \},
\end{cases}$
we have,
\begin{align*}
\Ex[X_Q] &= \sum_{\{u,v\} \in \mathcal{E}_Q} \Ex[\tA_{uv}\tB_{wr} - \tA_{uv}\tB_{wr}] \\
&= \sum_{\{u,v\} \in \mathcal{E}_Q} R_{uv} \Lambda_{uv} (1 - \Lambda_{uv}) \ge (n_c - 1) (k_c + k_j) \epsilon \gamma.
\end{align*}
The remainder of the proof follows {\it mutatis mutandis}.
\end{proof}

\bibliographystyleapp{IEEEtran}

\bibliographyapp{../biblio.bib}

\end{document}